\documentclass[letterpaper]{article} 
\usepackage[preprint]{aaai2027}    
\usepackage[hyphens]{url}            
\usepackage{graphicx}                
\usepackage{natbib}                  
\usepackage{caption}                 
\usepackage{booktabs}

\usepackage{amsmath}
\usepackage{amssymb}
\usepackage{amsfonts}
\usepackage{amsthm}
\usepackage{dsfont}       
\usepackage{bm} 
\usepackage{nicefrac}
\usepackage{multirow}
\usepackage{xcolor}
\usepackage{tikz}
\usepackage{subcaption}

\usetikzlibrary{positioning,calc,fit,arrows.meta,backgrounds,%
                shapes.geometric,shapes.arrows}   

\usepackage{ifthen} 
\newboolean{showappendix} 
\setboolean{showappendix}{True} 
\definecolor{cGraph}{HTML}{2563EB}
\colorlet{cgraph}{cGraph}                  
\definecolor{cGraphF}{HTML}{DBEAFE}
\definecolor{cFeat}{HTML}{EA580C}
\definecolor{cFeatF}{HTML}{FFEDD5}
\definecolor{cInk}{HTML}{1E293B}
\definecolor{cGrid}{HTML}{94A3B8}
\definecolor{cPanel}{HTML}{F8FAFC}
\definecolor{cPanelB}{HTML}{E2E8F0}

\tikzset{
  gnode/.style={circle, draw=cGraph, line width=0.9pt, fill=cGraphF,
                minimum size=6mm, inner sep=0pt},
  fnode/.style={circle, draw=cFeat, line width=0.9pt, fill=cFeatF,
                minimum size=7mm, inner sep=0pt, font=\small\sffamily, text=cFeat},
  gedge/.style={draw=cGrid, line width=1.0pt},
  fedge/.style={draw=cFeat, line width=1.2pt},
  panel/.style={rounded corners=5pt, draw=cPanelB, line width=1pt, fill=cPanel},
  stage/.style={font=\sffamily\small\bfseries, text=cInk, anchor=north},
  flow/.style={-{Stealth[length=3.4mm]}, line width=1.4pt, draw=cInk},
}

\graphicspath{{media/}}

\title{Neural Message Passing on Structural Interaction Graphs for Fully-Inductive Graph Neural Networks}

\author{
    Omer Yom Tov\textsuperscript{\rm 1}\footnote{Corresponding author},
    Avigdor Gal\textsuperscript{\rm 1}
}
\affiliations{
    \textsuperscript{\rm 1}Faculty of Data and Decision Sciences,
    Technion - Israel Institute of Technology\\
    Haifa, Israel \\
    omer.y@campus.technion.ac.il, avigal@technion.ac.il
}

\newcommand{\G}{\mathcal{G}}
\newcommand{\A}{\mathbf{A}}
\newcommand{\X}{\mathbf{X}}
\newcommand{\R}{\mathbb{R}}
\newcommand{\W}{\mathbf{W}}
\newcommand{\V}{\mathcal{V}}
\newcommand{\E}{\mathcal{E}}
\newcommand{\Y}{\mathbf{Y}}
\newcommand{\N}{\mathcal{N}}
\newcommand{\ypred}{\hat{\mathbf{Y}}}
\newcommand{\x}{\mathbf{x}}

\newcommand{\GS}{\mathcal{G}_{\mathrm{struct}}}
\newcommand{\MS}{\mathcal{M}_{\mathrm{struct}}}

\newcommand{\SIGIL}{\textsc{Sigil }}
\newcommand{\SIGILns}{\textsc{Sigil}}
\newcommand{\SIGILNC}{\textsc{Sigil-nc }}
\newcommand{\SIGILNCns}{\textsc{Sigil-nc}}
\newcommand{\SIGILLP}{\textsc{Sigil-lp }}
\newcommand{\SIGILLPns}{\textsc{Sigil-lp}}

\newtheorem{proposition}{Proposition}
\newtheorem{theorem}{Theorem}
\newtheorem{lemma}{Lemma}

\usepackage{pifont}

\usepackage{graphicx}
\begin{document}
\maketitle

\begin{abstract}
A central obstacle in building graph foundation models is the input heterogeneity in terms of feature space  dimensionality, semantics, and structure. Such heterogeneity limits the capability of graph neural networks to generalize to new graphs with unseen feature spaces. We address the transferability challenge with \SIGILns, a framework that maps any attributed graph to a unified representation space of fixed dimension. Given a graph, \SIGIL lifts it to a structural interaction graph, where nodes are the input feature dimensions and weighted, typed edges encode feature alignment across multiple orders of the graph's connectivity. A relational message-passing network embeds each feature dimension into a shared space, transforming the original node features, of arbitrary dimensionality, into representations transferable to any downstream graph. By construction, \SIGIL is equivariant to permutations of nodes, feature dimensions, and labels. Additionally, when the input features are one-hot indicators of discrete relations, \SIGIL recovers and strictly generalizes existing foundation models for knowledge graph reasoning. A single \SIGIL model, pretrained on one graph, delivers strong fully-inductive link prediction. Also, \SIGIL can be used to implement existing knowledge graph foundation models. As such, \SIGIL unifies several existing regimes in graph foundation model design under a single framework.
\end{abstract}

\section{Introduction}\label{sec:intro}
 \input{figures/framework}
Foundation models that perform inference on unseen tasks without task-specific training have reshaped machine learning in language and vision, yet graph-structured data has remained largely outside their reach~\cite{zhao2025graphany}. While a foundation model for text can encode any input through a shared, precomputed token vocabulary, no analogous universal vocabulary exists for graphs. Each graph arrives with its \emph{own} node-feature space, different in dimensionality, in the semantics of each coordinate, and in the label set to be predicted, so a model trained on one graph cannot, in general, even \emph{read} another. A \emph{fully inductive} graph model must therefore generalize across three axes at once: unseen graph structure, unseen feature spaces, and unseen label spaces~\cite{zhao2025graphany, finkelshtein2026equivariance}.
 
Two largely separate lines of work have made progress on this problem, each capturing part of what a graph foundation model requires. The first targets \emph{node classification} on arbitrary graphs. GraphAny~\cite{zhao2025graphany} casts inference as the analytical solution of a set of linear GNNs fused by a permutation-invariant attention, and subsequent models obtain unified node representations through tabular foundation models~\cite{hayler2025of,eremeev2026graphpfn}, learned view spaces~\cite{lee2026view}, or feature-space bridges~\cite{eliasof2026bridging}. These methods handle arbitrary feature and label spaces and respect the symmetries the setting demands, but they are designed for node classification alone and do not extend to link prediction or to the relational reasoning that structured graphs require. The second line builds \emph{knowledge graph foundation models} (KGFMs)~\cite{galkin2023ultra, zhang2024trix, lee2023ingram, geng2023relational}, which achieve striking zero-shot transfer to knowledge graphs with entirely unseen entities and relations by constructing a \emph{relation graph}, a graph whose nodes are relation types and whose edges record how relations co-occur in the structure, and learning transferable representations on it. This relational construction is powerful and principled, but it is fundamentally \emph{discrete}. It presumes a finite vocabulary of relation types and has no mechanism for continuous node or edge features, which excludes input graphs with continuous feature spaces.

These two research directions share a deeper commonality that, to our knowledge, has not been explicitly formalized in prior work. A KGFM's relation graph is exactly a structure built over the \emph{feature space} of a knowledge graph, namely its discrete relation types, rather than over nodes. Viewed this way, the relation graph is the discrete instance of a more general object, a graph over feature \emph{dimensions} whose edges encode how those dimensions interact through the connectivity of the underlying graph. If such an object could be constructed for \emph{continuous} features, it would unify the two research directions mentioned above. It inherits the relational reasoning of KGFMs while accepting the arbitrary feature spaces that node-classification Graph Foundation Models (GFMs) target, and it would apply to node classification and link prediction alike.

 We realize this object in \SIGIL (\textbf{S}tructural \textbf{I}nteraction \textbf{G}raphs for \textbf{I}nductive \textbf{L}earning) (see Figure~\ref{fig:pipeline} for overview). Given any graph (Figure~\ref{fig:pipeline}(A)), \SIGIL lifts it to a \emph{Structural Interaction Graph} (SIG): a graph whose nodes are the input feature dimensions and whose weighted, typed edges measure how pairs of features align across multiple orders of the graph's connectivity. Concretely, \SIGIL propagates features over increasing neighborhood orders, contrasts the features of adjacent nodes through an interaction operator, and summarizes the resulting patterns into a per-order Gram matrix over feature dimensions; stacking these orders yields the SIG (Figure~\ref{fig:pipeline}(B)). A relational message-passing network then runs \emph{on the SIG} to embed every feature dimension into a fixed-size space, which in turn maps the original graph, of arbitrary feature dimensionality, into a single unified representation (Figure~\ref{fig:pipeline}(C)). Because the parameters of this network act only on the fixed hidden width and never on the input dimension, one trained \SIGIL model applies to any graph (Figure~\ref{fig:pipeline}(D)). Crucially, when the input features are the one-hot indicators of discrete relations, the SIG construction recovers exactly the relation graph of a KGFM, so \SIGIL is a strict generalization of KGFMs from discrete to continuous input spaces.
 
 
In summary, we make the following contributions:
\begin{itemize}
\item We introduce SIG, a structural interaction graph that lifts any attributed graph to a graph over its feature dimensions, and show that it strictly generalizes the relation graph of KGFMs from discrete relation vocabularies to continuous feature spaces (\S~\ref{sec:sigil}, \S~\ref{sec:prop}).
\item We propose \SIGILns, a single framework that maps graphs of arbitrary feature and label spaces to a unified representation and addresses both node classification and link prediction, and we show it is equivariant to node, feature, and label permutations (\S\ref{sec:prop}).
\item We evaluate \SIGIL across three regimes that the literature has treated separately: 11 link prediction benchmarks (7 attributed, 4 non-attributed), inductive knowledge graph foundation models' suite, and 26 node classification benchmarks.
A single Cora-pretrained \SIGIL is the strongest fully-inductive link predictor on every attributed benchmark and remains competitive on non-attributed graphs. Also,
\SIGIL closely reproduces the performance of existing KGFMs when instantiated on discrete relations and admits more expressive variants. Finally, it is competitive with node-classification specific GFMs, ranking best or second-best on 10 of 26
datasets (\S\ref{sec:experiments}).
\end{itemize}
\section{Preliminaries}\label{sec:preliminaries}

\paragraph{Graph Notations.} Let $\G = (\V, \E)$ be a graph, with $n=|\V|$ nodes and $m=|\E|$ edges, where $\V$ and $\E$ represent the node and edge sets, respectively. Let $\A\in\{0,1\}^{n\times n}$ be the adjacency matrix, where $\A_{ij}=1$ if and only if $(i,j)\in \E$. Let $\X \in \R^{n\times d}$ be the node feature matrix. 
We use $\X_{i,:}$ to denote the feature vector of the $i$-th node, and $\X_{:,f}$ to denote the feature values of the $f$-th feature across all of the nodes.

\paragraph{Fully Inductive Graph Inference.} Graph inference typically involves two task types. First, in a \emph{link prediction} task, we assume that the full graph is not observable, and that there are missing links (edges) from $\E$. For every node $v$ in a query set $\V_{query}\subseteq\V$ we aim to rank the nodes of the candidate set $\{u\in\V \mid (v,u)\notin\E\}$ on the likelihood of $(v,u)$ to join $\E$, based on the observable graph $\G$ and node features $\X$. Second, in a \emph{node classification} task, we assume the existence of node labels $\Y \in \{1,2,...,C\}^n$ from $C$ distinct classes. At both train and inference time, only a subset of labels $\Y_{train}$ is available for a subset of nodes $\V_{train}\subset \V$, and the goal is to provide a prediction $\ypred$ for the missing labels $\Y_{test}$ of $\V_{test} = \V\setminus\V_{train}$. Traditionally, both tasks are modeled in a transductive learning setup, where a GNN is trained on the observable training data of $\G$, then inference is performed on the unobservable edges or node labels. Such setup usually does not allow generalization for new graphs. In a \emph{fully inductive} setup, A GNN is expected to perform inference over graphs that were not seen before, and which node feature dimensions/distributions might differ from what was processed in training.

\paragraph{Message Passing Neural Networks.} Message passing neural networks (MPNNs) \cite{pmlr-v70-gilmer17a} is a subclass of GNNs that learn latent node representations through an iterative update of a node $v$'s hidden state with the hidden states of its neighbors $\N_v$. Let $h_v^{(t)} \in \R^{d}$ be the hidden state of $v$ in iteration $t$, an MPNN updates the hidden state as follows.
\begin{align}
    h_v^{(t+1)} = \textsc{Upd}(h_v^{(t)}, \textsc{Agg}(\{\!\!\{ h_u^{(t)}\mid u\in\N_v \}\!\!\}))
\end{align}
where \textsc{Agg} is a learnable aggregation function, aggregating the hidden states of $\N_v$, and \textsc{Upd} is a learnable update function, which outputs a new hidden state for $v$ based on its previous state and the aggregated results of its neighbors' hidden states.

\section{Message Passing on Structural Interaction Graphs}\label{sec:sigil}

Our goal is to design a fully inductive GNN that supports inference capabilities on any new graph, with arbitrary feature and label dimension. In this section, we outline \SIGILns, a framework for designing such GNNs for both link prediction and node classification tasks. \SIGIL allows generalization to unseen feature and label spaces through learning representations on SIG, a graph lifted from the input graph $\G$, in which the node feature dimensions $[d]$ are the nodes, and the edges represent the alignment of features with respect to various views of $\G$'s connectivity through weights and types.

\subsection{Constructing Structural Interaction Graphs}\label{sec:construction}
To construct the SIG of a graph $\G$, we begin by creating multiple orders for the node features $\X$ by performing multiple non-parameterized feature propagations on $\G$. The $k$-th order feature matrix $\X^{(k)}\in\R^{n\times d}$ is obtained via:
\begin{equation}
    \X^{(k)} = \A^k\X
\end{equation}
Then, within each order $k$, we create edge features with an interaction operator $g: \R^d\times \R^d \rightarrow \R$ ({\em e.g.}, subtraction or element-wise multiplication). For each edge $(i,j)\in\E$, we create an edge feature $\x_{ij}^{(k)} = g(\X_{i,:}^{(k)},\X_{j,:}^{(k)})$. We use these intermediate edge features to construct a new node feature matrix $\Bar{\X}^{(k)}$, which satisfies
\begin{equation}
    \Bar{\X}_{i,:}^{(k)} = \frac{1}{|\N_i|}\sum_{j\in\N_i} \x_{ij}^{(k)}
\end{equation}
The resulting $\Bar{\X}^{(k)}$ columns characterize how feature values instantiate across nodes with respect to local neighborhoods. Depending on the choice of $g$, they capture how a node $v$ tend to be similar/different to its neighbors through the feature. To model how aligned these patterns are between two features, we use the Gram matrix on the columns
\begin{equation}
    \Bar{\A}^{(k)} = \left(\Bar{\X}^{(k)}\right)^{\top}\Bar{\X}^{(k)} \in \R^{d\times d}
\end{equation}
The resulting matrix $\Bar{\A}^{(k)}$ is a weighted adjacency matrix of a complete, undirected graph over the feature dimensions $[d]$, where the weight between features $f, f'$ measures the alignment of their local patterns, based on the characterization vectors $\Bar{\X}^{(k)}_{:,f}, \Bar{\X}^{(k)}_{:,f'}$. Finally, stacking the weighted adjacency matrices $\Bar{\A}^{(0)}, \Bar{\A}^{(1)}, ..., \Bar{\A}^{(k)}$ along the order axis, yields the structural interaction graph $\GS \in \R^{d\times d\times k}$, where edges are typed based on the feature order from which they  originate, and weighted according to the feature alignments within the order.
\paragraph{Extension for Directed Graphs.}
When $\G$ is directed, a node participates in an edge either as its
source or as its target, each carrying different information. We
therefore split the neighborhood aggregation by direction. Writing
$\N_i^{h} = \{j : (i,j)\in\E\}$ and $\N_i^{t} = \{j : (j,i)\in\E\}$,
we form two per-node summaries instead of one:
\begin{align}
    \Bar{\X}^{(k)}_{h_{i,:}} &= \frac{1}{|\N_i^{h}|}\sum_{j\in\N_i^{h}} \x_{ij}^{(k)},
    &
    \Bar{\X}^{(k)}_{t_{i,:}} &= \frac{1}{|\N_i^{t}|}\sum_{j\in\N_i^{t}} \x_{ji}^{(k)} 
\end{align}
Row $i$ of $\Bar{\X}^{(k)}_{h}$ (resp. $\Bar{\X}^{(k)}_{t}$) records how
feature dimensions instantiate on the edges that leave (respectively enter) node $i$.
The Gram summarization is then applied both \emph{within} and \emph{across} the two
channels, yielding four feature-alignment matrices per order rather than one, as follows.
\begin{equation}
    \Bar{\A}^{(k)}_{cc'} = \left(\Bar{\X}^{(k)}_{c}\right)^{\top}\Bar{\X}^{(k)}_{c'}
    \in\R^{d\times d},
    \quad c,c' \in \{h,t\} 
\end{equation}
The slice $\Bar{\A}^{(k)}_{cc'}$ measures how strongly feature $f$, seen through
the $c$ channel, aligns with feature $f'$, seen through the $c'$ channel, at a shared
node. Stacking over orders and channel pairs gives the binary-motif SIG
$\GS \in \R^{d\times d\times 4k}$, whose edge types are now indexed by an
(order, channel-pair) tuple. The undirected construction of the previous paragraph
is the special case in which the two channels coincide, collapsing the four slices
per order back to one. For simplicity, we formalize the architectures on the undirected case, treating the directed case as an immediate generalization. We provide detailed implementation of the construction pipeline via matrix multiplications in Appendix~\ref{sec:sig}.\footnote{All appendices are part of the supplementary material.}

\subsection{Learning Representations on the SIG}\label{subsec:sigil-rep}
The task of fully-inductive node representation learning (FI-NRL), formalized by~\cite{lee2026view}, requires outputting node representation without gradient-based training for graphs of arbitrary dimensionalities. We use pretrained GNNs with weights fitted for the SIG graph vocabulary to allow \SIGIL to infer informative node representations for the graphs without training on them first.
\paragraph{Relational Message Passing on the SIG.} Following the construction of $\GS$, \textsc{Sigil} embeds each feature $f\in[d]$ to a fixed-sized representation $\theta_f\in\R^{d_h}$ through running a relational MPNN over $\GS$. We initialize the node feature embeddings $\Theta^{(0)}$ according to an initialization policy ({\em e.g.}, Laplacian Eigenmaps), then update the representation with relational message passing GNN $\MS$ \cite{schlichtkrull2018modeling} over the order typed edges, as follows.
\begin{align}
    \Theta^{(t+1)} = \Theta^{(t)}\W_{self} + \sum_{k=0}^K \Bar{\A}^{(k)}\Theta^{(t)}\W_k
\end{align}
where the initial encoding and parameters $\W_i \in \R^{d_h\times d_h}$ map the features to a fixed-size dimension $d_h$. The final layer $\Theta^{(T)} \in \R^{d\times d_h}$ yields feature embeddings that are learned from local patterns and feature interactions of $\G$, and are used to reshape the original node features $\X\in \R^{n\times d}$ to the new, fixed dimension $d_h$ via:
\begin{align*}
    \mathbf{H} = \X \Theta^{(T)} \in \R^{n\times d_h}
\end{align*}
We can construct $\GS$ for any graph, with arbitrary number of nodes, edges or feature dimensions $d$. Also, the parameters of $\MS$ are completely independent of $d$ and  output new node features in a fixed dimension $d_h$. Therefore, \textsc{Sigil} effectively transforms every possible graph input, with arbitrary dimensions, to the same unified input space that is based both on the original node features and the graph structure. Moreover, we can create additional views of the node embeddings of the propagated node features via $\mathbf{H}^{(k)} = \X^{(k)}\Theta^{(T)}$, leveraging additional signals in downstream tasks.  




\subsection{Using \SIGIL for Downstream Tasks}

\paragraph{Designing Fully-Inductive Link Prediction Models.}
Link prediction admits a natural end-to-end fully-inductive design under
\SIGIL, since its output space does not depend on any per-graph vocabulary of
labels: the same scoring function transfers to any target graph. We exploit this
by pairing the SIG embeddings with expressive link prediction
GNNs~\cite{zhu2021neural, chamberlain2023graph, wang2024neural} as the
architectures that train on the \SIGIL embedding space, which allows them to
inference in zero-shot on any new graph.

\paragraph{Decoding Labels for Node Classification.}
Node classification introduces an additional axis of transfer: the label space
itself varies across graphs, and \SIGIL supports several decoding strategies
on top of its unified representation. We can decode the test node labels in an in-context manner
through closed-form analytical
solutions~\cite{zhao2025graphany}, which satisfies label equivariance. An extension
of the in-context approach, which preserves fully-inductive end-to-end design,
is to decode with tabular foundation
models~\cite{eremeev2026graphpfn, choi2026learning}. Alternatively, we can use
the node representations from \SIGIL and fit cheap linear/MLP layers to the
train-set nodes~\cite{eliasof2026bridging, lee2026view},
which makes \SIGIL a node encoder.

\section{Properties of \textsc{Sigil}}\label{sec:prop}
In this work, we establish two properties that situate \SIGIL within the
graph foundation model landscape. First, we show that \SIGIL is a strict
generalization of KGFMs: when the input
signal is the one-hot indicator of discrete relations, the SIG construction
recovers the relation graph of a KGFM (\S\ref{sec:kgfm}). Second, we show
that the node representations produced by \SIGIL satisfy the
symmetries identified as necessary for generalization in the fully
inductive regime~\cite{zhao2025graphany, finkelshtein2026equivariance, zhou2023double}. Specifically, the node representations produced by \SIGIL are equivariant to node permutations, invariant to feature dimension permutations, and, when composed
with an equivariant decoder, are also equivariant to label permutations. Complete proofs are given in Appendix~\ref{sec:proofs}.

\subsection{Generalizing KGFMs to Continuous Input Spaces}
\label{sec:kgfm}

\paragraph{The KGFM relation graph.}
 Let $\G=(\V,\E,\mathcal{R})$ be a knowledge graph with a finite relation vocabulary $\mathcal{R}$,
where every edge $(i,j)\in \E$ carries a relation type $r(i,j)\in \mathcal{R}$. KGFMs such
as ULTRA~\cite{galkin2023ultra} build a \emph{relation graph} $\G_{\mathrm{rel}}$ whose nodes are the $|\mathcal{R}|$
relation types and whose typed edges record how pairs of relations co-occur around shared entities, using them to learn transferable relation representations. ULTRA distinguishes four \emph{fundamental
relation interactions} according to whether the shared entity is a head or a
tail of each relation: head-to-head (h2h), tail-to-tail (t2t), head-to-tail
(h2t), and tail-to-head (t2h). A relational GNN over $\G_{\mathrm{rel}}$ then yields a
representation of each relation that is a function of these interactions alone,
and therefore transfers to any unseen knowledge graph.

The key observation is that each fundamental interaction is a \emph{product of relation-incidence matrices}. Define the head- and tail-incidence matrices
$M_h, M_t \in \{0,1\}^{n\times |R|}$ by
\begin{align*}
    & (M_h)_{v,r} = \mathds{1}\!\left[\exists\, (v,u)\in E : r(v,u)=r\right], \\
    & (M_t)_{v,r} = \mathds{1}\!\left[\exists\, (u,v)\in E : r(u,v)=r\right]
\end{align*}
Row $v$ records which relation $v$ participates as a head
(tail, respectively). Then the four interaction adjacencies over the relation set
$\mathcal{R}$ are, up to normalization,
\begin{align}\label{eq:matrices}
  & \A_{h2h} = M_h^\top M_h,\quad
  \A_{t2t} = M_t^\top M_t, \nonumber\\ 
  & \A_{h2t} = M_h^\top M_t,\quad 
  \A_{t2h} = M_t^\top M_h .
\end{align}

\paragraph{Recovering the relation graph as a SIG.}
We now show that the SIG construction of \S\ref{sec:construction}, instantiated
on one-hot relation features with the directed in/out separation, produces the matrices of Eq.~\ref{eq:matrices}.
Encode the relational signal of $G$ as one-hot edge features,
\[
  \x_{ij} = e_{\,r(i,j)} \in \{0,1\}^{d}, \quad d = |\mathcal{R}|
\]
so that the feature dimensions $[d]$ are in bijection with the relations $R$.
The SIG is then built over the relation vocabulary. Because $G$ is directed, we
apply the in/out variant of the construction (\S\ref{sec:construction},
following~\ref{sec:sig}): at each node we aggregate incident edge features
separately over out-edges and in-edges, giving two per-node summaries
$\bar{\X}_{h}, \bar{\X}_{t}$. With one-hot edge features and
mean (or sum) aggregation, these summaries are exactly the row-normalized
head/ tail-incidence matrices $M_h, M_t$.

\begin{theorem}[\SIGIL recovers the KGFM relation graph]
\label{prop:kgfm}
Let $\G=(\V,\E,\mathcal{R})$ be a knowledge graph encoded with one-hot relation edge
features $\x_{ij}=e_{r(i,j)}$. Then, with the directed in/out variant of the SIG
construction and multiplication, the slices of $\GS$
coincide with the four fundamental relation-interaction adjacencies of Eq.~\ref{eq:matrices}.
Consequently, relational message passing on $\GS$
(\S\ref{subsec:sigil-rep}) instantiates the relation encoder of a KGFM, and \SIGIL applied
to $\G$ reduces to a KGFM.
\end{theorem}


Theorem~\ref{prop:kgfm} places KGFMs as the discrete special case of \SIGIL
along the two axes of continuity and order. Replacing one-hot indicators with
arbitrary continuous features leaves every step of the construction well defined:
the incidence matrices become real-valued feature summaries and the Gram slices
become continuous feature-alignment matrices rather than integer co-occurrence
counts. \SIGIL therefore accepts the continuous node and edge features that
KGFMs cannot represent. Additionally, ULTRA's relation graph is built from
entity-sharing motifs, {\em i.e.}, relations that meet at a common node (a single hop
in $\G$). \SIGILns's multi-order propagation $\X^{(k)}=\A^k \X$ produces a slice
$\bar{\A}^{(k)}$ for each order $k$, so for $k>1$ the SIG additionally encodes
\emph{higher-order} interactions between features/relations that are not
adjacent but are connected through length-$k$ walks. Thus, \SIGIL strictly
generalizes KGFMs even when restricted to discrete relations.


\subsection{Equivariance to Node, Feature and Label Permutations}

We now formalize the symmetries of \SIGILns. Let $P\in\{0,1\}^{n\times n}$,
$Q\in\{0,1\}^{d\times d}$ be permutation matrices
acting on node feature dimensions, respectively. A relabeling of
the input by $(P,Q)$ transforms the graph as follows.
\begin{equation}
  \A \;\mapsto\; P\A P^\top, \quad \X \;\mapsto\; P\X Q^\top 
\end{equation}
We first
show that the SIG is \emph{invariant} to node permutations and
\emph{equivariant} to feature permutations, a property that makes it a
transferable object, and then propagate this through the model.

\begin{lemma}[Symmetries of the SIG]
\label{prop:sig}
For every order $k$, the SIG slice transforms as
$\bar{\A}^{(k)} \mapsto Q\,\bar{\A}^{(k)}Q^\top$ under the input relabeling
$(P,Q)$. In particular, $\bar{\A}^{(k)}$ is invariant to node permutations $P$
and equivariant to feature permutations $Q$.
\end{lemma}

\begin{proposition}[Node equivariance and feature invariance of \SIGILns]
\label{prop:equivariance} Assume that initialization is equivariant to node permutations, {\em i.e.}, $\Theta^{(0)} \mapsto Q\Theta^{(0)}$ for input relabeling $(P,Q)$. Then,
under the input relabeling $(P,Q)$, the unified node representation $\mathbf{H} = \X\Theta^{(T)}$
is equivariant to node permutations $P$ and invariant to feature permutations $Q$.
\end{proposition}

Proposition~\ref{prop:equivariance} shows that the internal feature embeddings
$\Theta^{(T)}$ are feature-\emph{equivariant} ($\Theta^{(T)}\mapsto Q\Theta^{(T)}$), while
the output node representation $H$ is feature-invariant: it does
not depend on the ordering or identity of the input feature dimensions at all.
This is exactly the requirement for cross-graph transfer articulated
by~\cite{zhao2025graphany, finkelshtein2026equivariance}, and it is what allows a single trained \SIGIL model
to read a graph whose $d$ features it has never seen. 

For the task of link prediction, 
symmetries hold for the query-conditioned construction used in KGFMs \cite{galkin2023ultra, zhang2024trix, huang2025how}, which maintains the reduction argument in Proposition~\ref{prop:sig}, and is in line with the symmetries established by~\cite{zhou2023double} for generalization in link prediction.

For the task of node classification, pairing \SIGIL with any equivariant decoder ({\em e.g.}, the closed form analytical solutions used by~\cite{zhao2025graphany}), ensures label equivariance for \SIGIL in node classification tasks. 

\section{Experiments}\label{sec:experiments}
We evaluate \SIGIL on two main tasks, namely link prediction and node classification. For link prediction, we design \SIGILLPns, a fully-inductive, zero-shot model that leverages both graph structure and node features through the SIG embeddings and directly subsumes KGFMs (as discussed in \S\ref{sec:kgfm}). This is enabled by the SIG embeddings being of fixed dimension, which is a property that, to the best of our knowledge, is a unique property of \SIGIL among all existing GFMs. For node classification, we design \SIGILNCns, which is a variant that outputs fixed-dimension node embeddings for an arbitrary graph, and then fit a downstream MLP to adapt to the output space of the task. While not being fully-inductive, this paradigm still exhibits fully-inductive representation learning~\cite{lee2026view},  positioning \SIGILNC as a node encoder.
These three experiments are three instantiations of a single
construction, chosen to unify regimes that the GFM literature has so far treated
separately, and are feasible by \SIGIL via the flexibility of modeling offered by the SIG and its inherent ability to encode both structure and feature spaces to fixed-size node embeddings.
Throughout, a single \SIGIL model is trained on one source graph and applied zero-shot to unseen targets. We instantiate \SIGILLP for link prediction and \SIGILNC for node classification. We compare \SIGIL against fully-inductive methods. Comparisons against transductive models trained directly on each target graph are given for the sake of completion in Appendix~\ref{sec:main_results} and design-choice ablations are deferred to Appendix~\ref{sec:more_experiments}. All experiments were run on two Nvidia A6000 GPUs ($48$GB of RAM for each) with torch-geometric~\cite{fey2019fastgraphrepresentationlearning}; full setups and hyperparameters are in Appendix~\ref{sec:hyperparams}.

\subsection{Link Prediction with Continuous Feature Spaces}\label{sec:lp_attr}
 
\paragraph{Datasets and baselines.} We evaluate \SIGILLP on 11 benchmarks in two groups.  The first comprises
seven \emph{attributed} graphs: citation networks~\cite{3045390.3045396} (Cora, CiteSeer, PubMed), co-purchase graphs~\cite{shchur2019pitfallsgraphneuralnetwork}
(AmazonComputers, AmazonPhotos), and co-authorship graphs~\cite{shchur2019pitfallsgraphneuralnetwork} (CoauthorCS,
CoauthorPhysics), following~\cite{dong2024pure, wang2024neural}. This is the regime \SIGILLP targets, and the one in
which, to the best of our knowledge, there is no existing fully-inductive link predictor that utilizes node features.
The second comprises four \emph{non-attributed} graphs from the benchmark of
\cite{NEURIPS2018_53f0d7c5, liao2026tfmlinkeruniversallinkpredictor}: C.ele, USAir, PB and NS, which carry no innate node features; there we instantiate $\X$ with DeepWalk embeddings \cite{Perozzi_2014} calculated on the train edges for \SIGILLP initializations, so that all methods operate on structural signal alone. We evaluate against two fully-inductive baselines: UniLP~\cite{dong_universal_2025} and TFMLinker~\cite{liao2026tfmlinkeruniversallinkpredictor} and against 4 transductive baselines: NBFNet~\cite{zhu2021neural}, ELPH~\cite{chamberlain2023graph}, NCNC~\cite{wang2024neural} and MPLP+~\cite{dong2024pure}. Evaluations against transductive baselines (except NBFNet which is \SIGILLP link prediction backbone) are in Appendix~\ref{sec:main_results}. Detailed experimental setup is in Appendix~\ref{sec:hyperparams}.
 
\begin{table*}[h!]
  \centering
  \footnotesize
  \setlength{\tabcolsep}{1.5pt}
  \begin{tabular}{@{}lccccccccccc@{}}
    \toprule
    & \multicolumn{7}{c}{\textbf{Attributed}} & \multicolumn{4}{c}{\textbf{Non-attributed}} \\
    \cmidrule(lr){2-8} \cmidrule(lr){9-12}
    Model & Cora & CiteSeer & PubMed & AmzComp & AmzPhoto & CoCS & CoPhys & C.ele & USAir & PB & NS \\
    \emph{Metric} & \emph{H@100} & \emph{H@100} & \emph{H@100} & \emph{H@50} & \emph{H@50} & \emph{H@50} & \emph{H@50} & \emph{H@50} & \emph{H@50} & \emph{H@50} & \emph{H@50} \\
    \midrule
    NBFNet
      & 71.65\scriptsize{$\pm$}2.27
      & 74.07\scriptsize{$\pm$}1.75
      & 58.73\scriptsize{$\pm$}1.99
      & \underline{13.14\scriptsize{$\pm$}1.04}
      & 21.35\scriptsize{$\pm$}3.36
      & \underline{71.86\scriptsize{$\pm$}0.93}
      & \underline{66.71\scriptsize{$\pm$}4.39}
      & 38.00\scriptsize{$\pm$}4.92 & 52.71\scriptsize{$\pm$}5.27 & 7.17\scriptsize{$\pm$}0.08 & \underline{89.60\scriptsize{$\pm$}1.83} \\
    UniLP
      & \underline{74.01\scriptsize{$\pm$}2.58}
      & 67.12\scriptsize{$\pm$}1.94
      & \underline{59.09\scriptsize{$\pm$}0.87}
      & $\geq$ 24h
      & \underline{34.71\scriptsize{$\pm$}1.48}
      & 69.38\scriptsize{$\pm$}1.70
      & $\geq$ 24h
      & 65.20\scriptsize{$\pm$}4.40
      & 85.98\scriptsize{$\pm$}2.00
      & \textbf{48.14\scriptsize{$\pm$}2.99}
      & 89.09\scriptsize{$\pm$}2.05 \\
    TFMLinker
      & -- & -- & -- & -- & -- & -- & --
      & \underline{68.62\scriptsize{$\pm$}5.97}
      & \textbf{91.67\scriptsize{$\pm$}1.37}
      & \underline{47.75\scriptsize{$\pm$}2.76}
      & \textbf{90.46\scriptsize{$\pm$}2.43} \\
    \SIGILLP
      & \textbf{77.99\scriptsize{$\pm$}$1.72^{\dagger}$}
      & \textbf{72.95\scriptsize{$\pm$}1.65}
      & \textbf{60.01\scriptsize{$\pm$}1.02}
      & \textbf{26.25\scriptsize{$\pm$}1.70}
      & \textbf{37.15\scriptsize{$\pm$}1.55}
      & \textbf{75.96\scriptsize{$\pm$}1.18}
      & \textbf{70.28\scriptsize{$\pm$}2.54}
      & \textbf{71.14\scriptsize{$\pm$}3.24}
      & \underline{89.60\scriptsize{$\pm$}1.69}
      & 44.05\scriptsize{$\pm$}3.66
      & 89.38\scriptsize{$\pm$}1.69 \\
    \bottomrule
  \end{tabular}
  \caption{Link prediction against \textbf{fully-inductive} baselines (\SIGILLP is ours). \SIGILLP, UniLP and TFMLinker are fully-inductive and NBFNet is the (originally transductive )\SIGILLP backbone. Best in \textbf{bold}, second-best \underline{underlined}. $\geq$24h marks per-query inference over 24 hours; ``--'' marks results that were not producible due to lack of published code. $\dagger$ - \SIGILLP is pretrained on Cora.} 
\label{tab:lp_fi}

\end{table*}
\paragraph{Results on Attributed Graphs.}
Table~\ref{tab:lp_fi} reports Hits@100 on the citation graphs and
Hits@50 elsewhere. On the attributed block, \SIGILLP is the
strongest fully-inductive method on every dataset, including the six it
has never seen. Against UniLP it leads by 5.83 on CiteSeer, 2.44 on
AmazonPhotos and 6.58 on CoauthorCS, and ties within noise on PubMed.
The advantage is structural rather than incidental: UniLP transfers
across graphs by discarding the feature space and falling back on
connectivity alone, whereas \SIGILLP  reads the target graph's
feature space through the SIG embeddings, which are of fixed dimension
$d_h$ regardless of the target's input dimensionality. The advantage also
compounds with scale: UniLP's in-context inference exceeds a 24-hour
budget yet \SIGILLP  completes even with the heavy NBFNet head.
The backbone comparison measures the same quantity from the opposite
direction. \SIGILLP  improves over its own NBFNet backbone on six
of the seven attributed datasets, from 1.28 points on PubMed to 15.80 on
AmazonPhotos, even though NBFNet is trained directly on each target graph
while \SIGILLP  never sees it; CiteSeer is the sole exception,
where the per-target-trained backbone retains a 1.12-point edge. Since
NBFNet on homogeneous graphs does not consume node features, and the two
models are otherwise identical in their decoder, this margin isolates
what the SIG embeddings contribute.

\paragraph{Results on Non-attributed Graphs.}
Averaged over the block, \SIGILLP reaches
73.54 Hits@50 against 72.10 for UniLP and 74.62 for TFMLinker, placing it second
of the three fully-inductive methods and ahead of UniLP. \SIGILLP is also the most stable of the three: its mean standard deviation over the block is 2.5, against 2.86 for UniLP and 3.13 for TFMLinker, and it is the tightest of the three on
\SIGILLP wins on C.ele, is second on USAir and very competitive on NS. Its one clear shortfall is PB, the densest graph in the block, where the absence of features leaves only the structural regularities that
heuristic-driven predictors already capture directly. The backbone gap is starkest in this block: without the SIG embeddings, NBFNet trained directly on each target graph reaches 38.00 on C.ele and 52.71 on USAir, against 68.22 and 88.39 for \SIGILLP. We note here that TFMLinker, the leading model, is being powered by a large tabular foundation model. Hence, \SIGILLP is competitive while being more parameter efficient and training on less data overall.
 
\subsection{Implementing KGFMs with \SIGILLP}\label{sec:lp_kg}
 
\paragraph{Setup.} Theorem~\ref{prop:kgfm} predicts that instantiating the SIG on one-hot relation indicators at
order zero recovers exactly ULTRA's relation graph. To verify the prediction empirically, we seek to re-implement and reproduce ULTRA within \SIGILns. We pretrain \SIGILLP on FB15k-237~\cite{toutanova-chen-2015-observed}, WN18RR~\cite{3504035.3504256} and CoDEx-Medium~\cite{safavi-koutra-2020-codex}, and evaluate under ULTRA's published protocol on its two inductive groups, inductive $(e,r)$, with unseen entities and relations (23 graphs), and inductive $(e)$, with unseen entities (18 graphs), alongside the three pretraining graphs. We additionally sweep the SIG order $K$, which \S\ref{sec:kgfm} identifies as a strict extension of the relation graph beyond the single-hop motifs ULTRA encodes. Full setup in Appendix~\ref{sec:hyperparams} and complete, per dataset results are in Appendix~\ref{sec:main_results}.
 
\begin{table}[h!]
  
  \centering
  \footnotesize
  \setlength{\tabcolsep}{4pt}
  \begin{tabular}{@{}lcccccc@{}}
    \toprule
    & \multicolumn{2}{c}{\textbf{Inductive $e,r$}} & \multicolumn{2}{c}{\textbf{Inductive $e$}} & \multicolumn{2}{c}{\textbf{Pretraining}} \\
    \cmidrule(lr){2-3} \cmidrule(lr){4-5} \cmidrule{6-7}
    Model & MRR & Hits@10 & MRR & Hits@10 & MRR & Hits@10 \\
    \midrule 
    \textsc{Ultra} & \textbf{34.5} & \textbf{51.3} & \underline{43.1} & 56.6 & 40.7& 56.8  \\
    \SIGILns$(0)$ & \underline{32.6} & \underline{50.1} & 39.5 & 55.0 & 40.7 & 55.6 \\
    \SIGILns$(1)$ & 32.4 & 49.2 & \underline{43.1} & \underline{56.8} & \textbf{43.7} & \textbf{59.0} \\
    \SIGILns$(2)$ & 32.3 & 49.8 & \textbf{43.5} & \textbf{58.3} & \underline{43.6} & \underline{58.9} \\
    \bottomrule
  \end{tabular}
  \caption{KG reasoning results against ULTRA. ULTRA can be reproduced within the \SIGIL framework by treating KG relations as one-hot feature vectors. \SIGILns($k$) denotes a \SIGILLP model with $k$-th order SIG.}
    \label{tab:kg_experiment}

\end{table}
 
\paragraph{Results.} Table~\ref{tab:kg_experiment} reports MRR (Mean Reciprocal Rank) and Hits@10. \textsc{Sigil}(0) reproduces ULTRA across all three groups to within 2.1 MRR and 1.2 Hits@10, matching it exactly on pretraining MRR (40.7) and exceeding it on entity-inductive Hits@10. Because the relation graph is an actual SIG (Theorem~\ref{prop:kgfm}), these residuals are attributable to existing differences, such as more complex interaction operators and aggregations of \SIGIL, that can empirically improve expressivity, consistent with the findings of~\cite{zhang2024trix}. Higher orders, implemented via sweeps on $K$ values, behave in a regime-dependent way. Relative to \SIGILns(0), order 2 gains MRR on the entity-inductive group and on the pretraining group, but loses MRR on the fully-inductive $(e,r)$ group. Hence, we do not claim that higher orders are uniformly preferable, consistent with \cite{huang2025how}. Therefore, greater
expressivity does not automatically yield better-performing KGFMs, and characterizing when it pays off is a natural extension of the analysis to the multi-order setting.
 
\subsection{Node Classification}\label{sec:nc}
 
\paragraph{Datasets and Baselines.} We follow GraphAny~\cite{zhao2025graphany} in dataset selection and splits, evaluating
\SIGILNC on 26 benchmarks whose feature dimensionalities and class counts
(2 to 70) vary widely. We use feature normalization and reduce input feature dimensions beyond 1024
with PCA, following~\cite{finkelshtein2026equivariance}. On featureless datasets from~\cite{Ribeiro_2017}, we initialize features with DeepWalk~\cite{Perozzi_2014}. We compare against four
fully-inductive GFMs designed specifically for node classification: GraphAny~\cite{zhao2025graphany}, TS-MEAN~\cite{finkelshtein2026equivariance}, TAG~\cite{hayler2025of} and RGVT~\cite{lee2026view}, and defer transductive baselines (MLP, GCN,
GAT) to Appendix~\ref{sec:main_results}. \SIGILNC is pretrained on Cora and applied to every target graph without gradient updates to the SIG encoder. A light MLP is then fitted on the target's training labels to map the unified representation into its label space. \SIGILNC therefore realizes fully-inductive representation learning in the sense of~\cite{lee2026view}, rather
than end-to-end fully-inductive classification. Full setup in Appendix~\ref{sec:hyperparams}.
 
\begin{table}[t]

  \centering
  \footnotesize
  \setlength{\tabcolsep}{4pt}
  \begin{tabular}{@{}lccccc@{}}
    \toprule
    & GraphAny & TS-MEAN & TAG & RGVT & \SIGILNC \\
    \midrule
    Avg.\ Acc.\ (\%) & 64.51 & 65.09 & \underline{71.57} & \textbf{71.78} & 66.10 \\
    \# Best (of 26)  & 0 & 0 & \textbf{11} & \underline{10} & 5 \\
    \# Top-2 (of 26) & 1 & 2 & \underline{17} & \textbf{22} & 10 \\
    \bottomrule
  \end{tabular}
  \caption{Node classification against fully-inductive baselines, aggregated over the 26 benchmarks; per-dataset results are in Appendix~\ref{sec:main_results} (Table~\ref{tab:main_nc}).}
  \label{tab:nc_summary}
\end{table}

\paragraph{Results.} Table~\ref{tab:nc_summary} reports average accuracy and the winning statistics of GFMs for node classification on the 26-dataset benchmark, with full experimental results available in Appendix~\ref{sec:main_results}. \SIGILNC does not match the specialized models: at 66.10
it trails RGVT (71.78) and TAG (71.57), while placing above GraphAny (64.51) and TS-MEAN (65.09), and is best on
5 of 26 datasets and best-or-second on 10, whereas GraphAny and TS-MEAN, within a point of it on average, win none and reach the top two on 1 and 2 respectively. \SIGILNC is therefore not uniformly mid-ranked but strong on a subset of graphs and weak on others, and its wins fall on graphs far from Cora in both
dimensionality and label space, which indicates that the SIG embeddings enable transferability to different domains.
We attribute the existing gap to the cost of unification.
Compressing arbitrary feature spaces into a fixed $d_h$ leaves the original coordinates individually inaccessible to the decoder, which consequently requires
more supervision than methods that retain the raw features. In Appendix~\ref{sec:more_experiments}, we vary the
number of labeled nodes per class to test this. 
Node classification is thus the weakest of our three regimes, yet the results still signal that learning on SIGs is a valid method for designing node classification KGFMs that can be competitive with existing, established methods.
\section{Related Work}\label{sec:related_work}

\subsection{Graph Foundation Models} \label{sec:GFM} GFMs are designed to learn representations that generalizes across various graph datasets, domains and dimensionalities, unlike GNNs~\cite{kipf2017semisupervised, hamilton2017inductive, velickovic2018graph}, which assume fixed feature orderings and dimensions, and are trained in a transductive setting. Several approaches have been proposed for designing GFMs for varying settings, domains, and tasks. These include aligning feature spaces through projections \cite{xia2024opengraph, xia2024anygraph}, connecting node features a graph structure \cite{frasca2024foundationmodelsgraphsanalysis, galkin2023ultra, zhang2024trix}, using prior-fitted networks and tabular foundation models \cite{eremeev2026graphpfn,choi2026learning,hayler2025of}, and utilizing LLMs for textual attributed graphs \cite{liu2024one}. A key open challenge, not fully addressed by existing approaches, is how to encode many graphs of different domains to a unified latent space. Unlike domains such as natural language processing, where precomputed static embeddings allows for any input to be encoded to the same space, in the graph domain this remains a challenge due to the heterogeneity of graph datasets and variance in their dimensionalities. In this work, we propose a framework that combines the graph structure and the signal from the node/edge features and outputs node embeddings in a unified, transferable input space that can support training of GNNs for zero-shot generalization. 

\subsection{Knowledge Graph Foundation Models}\label{sec:KGFM}
KGFMs are a family of GFMs designed for zero-shot reasoning on knowledge graphs (KGs), and can infer entities and relations that were unseen during their training by learning transferable relation embeddings from invariant graph motifs and relation interactions on the KG. Notable examples are ULTRA~\cite{galkin2023ultra}, RMPI~\cite{geng2023relational}, InGRAM~\cite{lee2023ingram} and TRIX~\cite{zhang2024trix}, all leveraging a new relation graph, which models co-occurrences of the relations (edge types) in graph motifs of the original KG. Despite the demonstrated generalization capabilities in the KG domain, these models are still limited to discrete features, and cannot be properly applied to graphs with continuous node/edge features. In this work, we propose a framework that generalizes to continuous features. By rethinking discrete relation types as one-hot edge feature vectors, we design a model that can be applied to a KG in a discrete case, but is also applicable to graphs with continuous input spaces. 


\section{Conclusions}\label{sec:conclusion}
We introduced \SIGILns, a framework that lifts any attributed graph to a structural interaction graph over its feature dimensions and learns transferable representations on it via relational message passing, mapping graphs of arbitrary dimensionality into a shared space of fixed width. The construction is equivariant to node, and feature permutations, and recovers the relation graph of a knowledge graph foundation model exactly when instantiated on one-hot relation indicators. Empirically, \SIGILLP is the strongest fully-inductive link predictor on every attributed benchmark and remains competitive on non-attributed graphs; the same construction reproduces ULTRA on its inductive benchmark, and yields a node encoder that ranks best or second-best on 10 of 26 classification benchmarks. Transferability, relational reasoning and edge-level tasks therefore arise from a single mechanism rather than from separate designs. For future work, designing and refining architectures for learning on SIGs beyond simple MPNNs is a natural next step, as well as studying the expressive power of SIGs, what information they preserve or lose and how it effects the SIG embeddings. 

\bibliography{aaai2027.bib}

\ifthenelse{\boolean{showappendix}}{
\onecolumn
\appendix

\section{Constructing The SIG}
\label{sec:sig}

The SIG is a deterministic, equivariant and parameter-free function of $(\X,\A)$, computed once per graph and
cached. Let $\rho$ denote row-wise $\ell_2$ normalization and
$\tau(\mathbf{M})=\mathbf{M}/\operatorname{tr}(\mathbf{M})$ trace normalization (identity when
$|\operatorname{tr}(\mathbf{M})|\le\varepsilon$).

\paragraph{Views.}
With $\hat{\A}=\mathbf{D}^{-\frac{1}{2}}\A\mathbf{D}^{-\frac{1}{2}}$ normalized and
self-loop-free --- so that order $0$ (ego) stays distinct from order $k$ (pure $k$-hop context):
\begin{equation*}
  \X^{(k)} \;=\; \rho\!\left(\hat{\A}^{k}\X\right)\in\mathbb{R}^{n\times d},
  \qquad k=0,\dots,K,
\end{equation*}
where the recursion runs on the unnormalized iterate and $\rho$ is applied only to the stored view.

\paragraph{Edge features.}
Let $\mathbf{B}_h,\mathbf{B}_t\in\{0,1\}^{n\times m}$ be the head and tail incidence matrices, defined as:
\begin{equation*}
    \mathbf{B}_{h_{ie}} = \mathds{1}[\exists v\in\V \ s.t\ e=(i,v) ],\quad \mathbf{B}_{t_{ie}} = \mathds{1}[\exists v\in\V \ s.t\ e=(v,i) ]
\end{equation*}
so
$\mathbf{B}_h^{\!\top}\X^{(k)}\in\mathbb{R}^{m\times d}$ gathers head-endpoint features. Per-edge
features are an entrywise map of the two endpoint blocks,
$\mathbf{F}^{(k)}=g(\mathbf{B}_h^{\!\top}\X^{(k)},\mathbf{B}_t^{\!\top}\X^{(k)})\in\mathbb{R}^{m\times d}$:
\begin{equation*}
  \quad
  \underbrace{\left|(\mathbf{B}_t-\mathbf{B}_h)^{\!\top}\X^{(k)}\right|}_{\textsc{absdiff}},
  \quad
  \underbrace{(\mathbf{B}_h^{\!\top}\X^{(k)})\odot(\mathbf{B}_t^{\!\top}\X^{(k)})}_{\textsc{hadamard}},
  \quad
  \underbrace{\left[\mathbf{B}_h^{\!\top}\X^{(k)}\,\big\|\,\mathbf{B}_t^{\!\top}\X^{(k)}\right]}_{\textsc{concat}}
\end{equation*}

\paragraph{Role aggregation.}
Each edge feature returns to its endpoints, separately by role, with degree-normalized mean
aggregation:
\begin{equation*}
  \Bar{\X}_h^{(k)}=\mathbf{D}_h^{-1}\mathbf{B}_h\mathbf{F}^{(k)},
  \qquad
  \Bar{\X}_t^{(k)}=\mathbf{D}_t^{-1}\mathbf{B}_t\mathbf{F}^{(k)}
  \;\in\;\mathbb{R}^{n\times d},
  \qquad
  \mathbf{D}_\ast=\max\!\big(\operatorname{diag}(\mathbf{B}_\ast\mathbf{1}_m),\mathbf{I}_n\big),
\end{equation*}
the clamp passing isolated nodes through as zero rows. Row $i$ of $\Bar{\X}_h^{(k)}$ is the mean
order-$k$ edge feature over the edges where $i$ is the head. This is the only step that separates
the two roles, and it is what makes the construction direction-aware.

\paragraph{Interactions and stacking.}
The SIG edge weights are the four role-crossed products, where multiplications cancels the node axis:
\begin{equation*}
  \Bar{\A}^{(k)}_{hh}=\tau\!\left(\Bar{\X}_h^{(k)\top}\Bar{\X}_h^{(k)}\right),\quad
  \Bar{\A}^{(k)}_{tt}=\tau\!\left(\Bar{\X}_t^{(k)\top}\Bar{\X}_t^{(k)}\right),\quad
  \Bar{\A}^{(k)}_{ht}=\tau\!\left(\Bar{\X}_h^{(k)\top}\Bar{\X}_t^{(k)}\right),\quad
  \Bar{\A}^{(k)}_{th}=\mathbf{\bar{A}}^{(k)\top}_{ht}
\end{equation*}
each in $\mathbb{R}^{d\times d}$. Entry $(f,f')$ of $\Bar{\A}^{(k)}_{ht}$ accumulates over all
nodes how strongly dimension $f$ in the head role aligns with dimension $f'$ in the tail role - 
this is the interaction the SIG encodes as an edge between feature nodes $f$ and $f'$. Stacking over
orders,
\begin{equation*}
  \GS=\operatorname{stack}_{k=0}^{K}\left(\Bar{\A}^{(k)}_{hh},\Bar{\A}^{(k)}_{tt},\Bar{\A}^{(k)}_{ht},\Bar{\A}^{(k)}_{th}\right)
  \in\mathbb{R}^{d\times d\times R},
  \qquad R=4(K+1),
\end{equation*}

\paragraph{Special cases.}
Given native edge features $\mathbf{F}$, the view and edge-feature steps are skipped and $\mathbf{F}$
enters role aggregation directly, giving an order-$0$ SIG with $R=4$. For a knowledge graph, taking
$\mathbf{F}=\mathbf{Z}\in\{0,1\}^{m\times|\mathcal{R}|}$, the one-hot relation indicator, makes the
rows of $\Bar{\X}_h=\mathbf{B}_h\mathbf{Z}$ and $\Bar{\X}_t = \mathbf{B}_t\mathbf{Z}$ the head/tail
relation-role incidences of each nodem and the resulting $\GS$ is exactly ULTRA's relation graph~\cite{galkin2023ultra}.

\section{Complexity Analysis for SIG Construction}\label{sec:complex}

Throughout, $n$ and $m$ are the node and edge counts of the input graph, $d$ the feature
dimension, $K$ the number of propagation orders, $R=4(K+1)$ the SIG relation count, and
$\bar{k}=m/n$ the average degree. The encoder that consumes the SIG has $T$ layers and width
$d_h$. The construction is deterministic and parameter-free, so it is computed once per graph and
cached; every cost below is paid once and amortized over all training steps on that graph.

\paragraph{Stage-by-stage cost.}
Each stage of \S\ref{sec:sig} streams over either the edge axis or the node axis:

\begin{center}
\begin{tabular}{llll}
  \toprule
  \footnotesize
  Stage & Time & Working memory & Retained \\
  \midrule
  Views $\X^{(k)}$, $k=0..K$                                  & $O(Kmd)$   & $O(nd)$  & $O(Knd)$  \\
  Edge features $\mathbf{F}^{(k)}$                            & $O(Kmd)$   & $O(md)$  & -       \\
  Role aggregation $\Bar{\X}^{(k)}_h,\Bar{\X}^{(k)}_t$    & $O(Kmd)$   & $O(nd)$  & -       \\
  Interaction operators $\Bar{\A}^{(k)}_\ast$               & $O(Knd^2)$ & $O(d^2)$ & $O(Kd^2)$ \\
  \midrule
  Total & $O\!\left(K(md+nd^2)\right)$ & $O(md+nd)$ & $O(K(nd+d^2))$ \\
  \bottomrule
\end{tabular}
\end{center}

\begin{figure}[h!]
    \centering
    \includegraphics[width=0.5\linewidth]{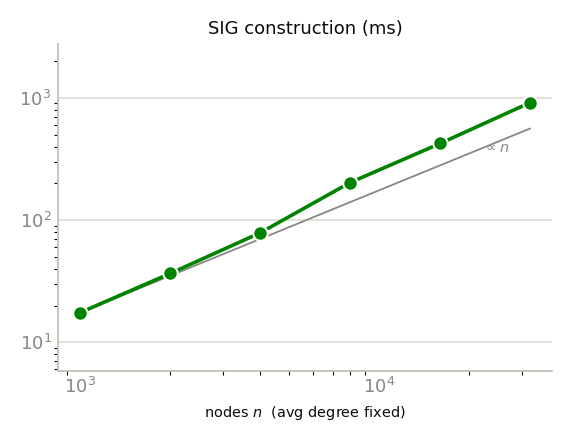}
    \caption{SIG construction time in ms vs. \#nodes (fixed average degree $\bar{k}$).}
    \label{fig:placeholder}
\end{figure}

The first three stages are all $O(Kmd)$: $\hat{\A}$ has $O(m)$ nonzeros, so each of the $K$
propagation steps costs $O(md)$. the entry-wise map $g$ touches each of the $m$ edges once per
order, and role aggregation is two sparse products $\mathbf{B}_\ast\mathbf{F}^{(k)}$ with $m$
nonzeros each, due to each edge being indicates once in every $\mathbf{B}_\ast$. 
Only the interaction stage leaves the edge axis: per order it forms three $d\times d$
products $\Bar{\X}^\top\Bar{\X}$ from $[d,n]\times[n,d]$ operands at $nd^2$ multiply, adds
each. We note here that by capping $d$ at 1024 via dimension reduction (like PCA), we can treat it as a constant, yet we maintain the analysis for the broader, more general case.

\section{Proofs}\label{sec:proofs}
\subsection{Proof of Theorem~\ref{prop:kgfm}}
\begin{proof}[Proof sketch]
By the argument above, the per-node in/out summaries equal $M_t$ and $M_h$
(up to degree normalization). The Gram summarization
$\bar{\A} = \bar{\X}^\top \bar{\X}$ applied \emph{within} and \emph{across} the
in/out channels yields the four products
$M_h^\top M_h,\, M_t^\top M_t,\, M_h^\top M_t,\, M_t^\top M_h$, which are exactly
$(\star)$. Stacking these as typed slices of $\GS$ reproduces
ULTRA's relation graph, and the relational MPNN $\MS$ of
\S\ref{subsec:sigil-rep} has the same functional form as the GNN ULTRA runs on that
graph. 
\end{proof}

\subsection{Proof of Lemma~\ref{prop:sig}}
\begin{proof}
The propagated features satisfy
$\X^{(k)} = \A^k \X \mapsto (P\A P^\top)^k (P\X Q^\top) = P \A^k P^\top P \X Q^\top
= P \X^{(k)} Q^\top$,
using $P^\top P = I$. Choosing an interaction operator $g$ that acts
coordinate-wise in the feature dimension (e.g., absolut difference or element-wise
product) gives $g(uQ^\top, vQ^\top) = g(u,v)Q^\top$, so each edge feature obeys
$\x^{(k)}_{ij}\mapsto \x^{(k)}_{ij}Q^\top$ with edges relabeled by the node
permutation. Neighborhood mean aggregation is node-equivariant, hence
$\bar{\X}^{(k)} \mapsto P\,\bar{\X}^{(k)}Q^\top$. Finally, the Gram summarization
cancels the node permutation:
\[
  \bar{\A}^{(k)} = \bar{\X}^{(k)\top}\bar{\X}^{(k)}
  \;\mapsto\;
  (P\bar{\X}^{(k)}Q^\top)^\top (P\bar{\X}^{(k)}Q^\top)
  = Q\,\bar{\X}^{(k)\top} P^\top P\, \bar{\X}^{(k)} Q^\top
  = Q\,\bar{\A}^{(k)}Q^\top 
\]
\end{proof}

\subsection{Proof of Proposition~\ref{prop:equivariance}}

\begin{proof}
Assume inductively that $\Theta^{(t)}\mapsto Q\,\Theta^{(t)}$. The
relational update
$\Theta^{(t+1)} = \Theta^{(t)}\W_{self}
  + \sum_{k} \bar{\A}^{(k)}\Theta^{(t)}\W_k$,
whose weights act on the fixed hidden dimension $d_h$, transforms as
\[
  \Theta^{(t+1)} \;\mapsto\;
  Q\Theta^{(t)}\W_{self}
  + \sum_k \big(Q\bar{\A}^{(k)}Q^\top\big)\big(Q\Theta^{(t)}\big)\W_k
  =  Q\Big(\Theta^{(t)}\W_{self} + \sum_k \bar{\A}^{(k)}\Theta^{(t)}\W_k\Big)
  = Q\,\Theta^{(t+1)}
\]
so $\Theta^{(T)}\mapsto Q\,\Theta^{(T)}$. Therefore
\[
  \mathbf{H} = \X\Theta^{(T)}
  \;\mapsto\;
  (P\X Q^\top)(Q\Theta^{(T)})
  = P \X Q^\top Q \Theta^{(T)}
  = P \X \Theta^{(T)}
  = P \mathbf{H}
\]
where $Q^\top Q = I$ eliminates the feature permutation. The auxiliary views
$\mathbf{H}^{(k)} = \X^{(k)}\Theta^{(T)}$ satisfy $\mathbf{H}^{(k)}\mapsto P \mathbf{H}^{(k)}$ by the same
argument. 
\end{proof}

\section{Detailed Experimental Setups}\label{sec:hyperparams}
\subsection{Experimental Setup for Link Prediction}

\paragraph{\SIGILLP Architecture.}
\SIGILLP is designed to inference \emph{conditioned} node representation as a direct generalization of KGFMs~\cite{galkin2023ultra, zhang2024trix, huang2025how}. After constructing $\GS$, we leverage two NBFNet modules: $\MS$ for $\GS$ and the $\mathcal{M}$ for $\G$. Given a query node $v$ for which we need to rank candidate nodes $u\in\V$, we initialize the SIG node embedding of the feature $f$ as:
\begin{align*}
    \theta_{f\mid v}^{(0)} = \X_{v,f}\cdot\mathbf{1}_{d_h} 
\end{align*}
which is the value of the $f$ for $v$, multiplied by an all-ones vector of dimension $d_h$. This initialization for every feature creates \emph{conditioned} feature embeddings $\Theta^{(T)}\mid v$, which in turns enables the downstream model $\mathcal{M}$ to learn node representations that are conditioned not only on the structure of $v$, but also on it's actual features. The rest of the pipelines follows as the main flow described in \S\ref{sec:sigil} - transforming the original node features to $\R^{d_h}$ via $\Theta^{(T)}$, which in turns enables $\mathcal{M}$ to learn in the weight space of $\R^{d_h\times d_h}$ regardless of the original graph's dimension. \SIGILLP is trained to minimize cross entropy loss of a positive links against a set of $n$ negative links:
\begin{align*}
    \mathcal{L} = -\log p((v,u)) - \frac{1}{n}\sum_{i=1}^n \log p((v_i', u_i'))
\end{align*}
Where $(v,u)\in \E$ is a positive link in the graph and $\{(v_i', u_i')\}_{i=1}^n$ are negative links that are not in the graph.
\begin{table}[h!]
\centering
\caption{SIGIL-LP architecture parameters for main results.}
\label{tab:lp-arch}
\begin{tabular}{llc}
\toprule
& \textbf{Hyperparameter} & \textbf{Value} \\
\midrule
\multicolumn{3}{l}{\emph{Shared}} \\
& $d_h$ & 128 \\
& Feature preprocessing & PCA to 1024, standardize, $\ell_2$ \\
& Edge interaction & \textsc{Absdiff} \\
\midrule
\multicolumn{3}{l}{\emph{SIG encoder}} \\
& Layers & 3 \\
& SIG channels $R$ & 4 (hh, tt, ht, th) \\
& Relation operator & DistMult (diagonal $w_r$ + shared linear) \\
& Normalization / residual & LayerNorm; short-cut + boundary re-injection \\
\midrule
\multicolumn{3}{l}{\emph{Link predictor}} \\
& Layers & 6 \\
& Message function & DistMult \\
& Aggregation & PNA \\
& Scorer & MLP $128 \to 128 \to 1$ (ReLU) \\
\bottomrule
\end{tabular}
\end{table}

\begin{table}[h!]
\centering
\caption{Training configuration for the Cora-trained SIGIL-LP checkpoint.}
\label{tab:lp-train}
\begin{tabular}{lc}
\toprule
\textbf{Hyperparameter} & \textbf{Value} \\
\midrule
Training graph & Cora \\
Optimizer & AdamW \\
Learning rate & $5\times10^{-4}$ \\
Steps & 2{,}000 \\
Batch size (positive edges) & 256 \\
Negatives per positive & 1 \\
Edge split (train/val/test) & 70\% / 10\% / 20\% \\
\bottomrule
\end{tabular}
\end{table}
\paragraph{Evaluation.} 
Our evaluation datasets and baselines are briefly discussed in~\ref{sec:lp_attr}. For these datasets, we largely follow a conventional evaluation protocol, in which each dataset is splits into train, validation and test links with sizes of $70\%/10\%/20\%$ of the total links in the graph. As \SIGILLP and the other fully-inductive baselines run zero-shot on the graph, the train set can be considered as a context set. For each positive test links, we evaluate its score as ranked against a set of negative links, equals in size to the total set of test positives. For the non-attributed datasets only and the fully-inductive methods only (\SIGILLP and UniLP), we use the validation set edges in test time as part of the train set as an additional context.

\subsection{Experimental Setup for KG Reasoning}
\paragraph{\SIGILLP Architecture.}
\SIGILLP architecture for the KG reasoning experiment is derived via the special case construction of the SIG, detailed in \S\ref{sec:sig}. As each edge feature is ultimately a one hot indicator of a discrete relation, the SIG is exactly the relation graph of a KGFM such as ULTRA. The parameter choices and training regime follows ULTRA to realize the equivalence and subsumption of ULTRA inside \SIGILns.
\begin{table}[t]
\centering
\caption{\SIGILLP architecture parameters for KG results. The SIG order $K$ changes per checkpoint.}
\label{tab:kgfm-arch}
\begin{tabular}{llc}
\toprule
& \textbf{Hyperparameter} & \textbf{Value} \\
\midrule
\multicolumn{3}{l}{\emph{Shared}} \\
& Hidden dimension $d_h$ & 64 \\
\midrule
\multicolumn{3}{l}{\emph{Relation-SIG encoder}} \\
& Layers & 6 \\
& SIG channels $R$ & $4(K{+}1) = 4$ (hh, tt, ht, th) \\
& Relation operator & DistMult (diagonal $w_r$ + shared linear) \\
& Normalization / residual & LayerNorm; short-cut + boundary re-injection \\
\midrule
\multicolumn{3}{l}{\emph{Link predictor}} \\
& Layers & 6 \\
& Message function & DistMult \\
& Aggregation & sum \\
& Scorer & MLP $64 \to 64 \to 1$ (ReLU) \\
\bottomrule
\end{tabular}
\end{table}

\begin{table}[t]
\centering
\caption{Pretraining configuration for \SIGILLP in the KG experiments.}
\label{tab:kgfm-train}
\begin{tabular}{lc}
\toprule
\textbf{Hyperparameter} & \textbf{Value} \\
\midrule
Pretraining mixture & FB15k-237, WN18RR, CoDEx-Medium \\
Optimizer & AdamW \\
Learning rate & $5\times10^{-4}$ \\
Steps & 20{,}000 \\
Batch size (query triples) & 8 \\
\bottomrule
\end{tabular}
\end{table}
\paragraph{Evaluation.}
Our evaluation setup is largely based on ULTRA's KG benchmark, which is based on 54 KGs divided into three sets - inductive $(e, r)$, inductive $(e)$ and transductive, and 3 pretraining graphs. We evaluate only on the first two + the pretraining graphs. The rest of our experimental setup, including train-test splits and positive-negative sampling is identical to the standard evaluation setups used by the KGFM literature~\cite{galkin2023ultra, zhang2024trix, huang2025how}.
\subsection{Experimental Setup for Node Classification}

\paragraph{\SIGILNC Architecture.}
\SIGILNC initialize feature embeddings on the SIG via laplacian eigenmaps, then runs the SIG encoder to produce feature embeddings $\Theta^{(T)}$ which are used to transform the original node features $\X\in\R^{n\times d}$ to $\R^{n\times d_h}$. At evaluation time, we fit a light MLP on the fully-inductive node embeddings from \SIGIL to the label space of the downstream task.
\begin{table}[h!]
\centering
\caption{\SIGILNC architecture parameters for main results.}
\label{tab:nc-arch}
\begin{tabular}{llc}
\toprule
& \textbf{Hyperparameter} & \textbf{Value} \\
\midrule
\multicolumn{3}{l}{\emph{Shared}} \\
& $d_h$ & 512 \\
& $K$ & 3 \\
& Edge interaction $g(x_h, x_t)$ & \textsc{Absdiff} \\
& Feature preprocessing & PCA to 1024, standardize, $\ell_2$ \\
\midrule
\multicolumn{3}{l}{\emph{SIG encoder}} \\
& Layers & 1 \\
& Relation operator & RGCN \\
\midrule
\multicolumn{3}{l}{\emph{Label decoder}} \\
& Head & 2-layer MLP \\
\bottomrule
\end{tabular}
\end{table}

\begin{table}[h!]
\centering
\caption{Training configuration for the Cora \SIGILNC
checkpoint.}
\label{tab:nc-train}
\begin{tabular}{lc}
\toprule
\textbf{Hyperparameter} & \textbf{Value} \\
\midrule
Training graph & Cora \\
Inner objective & MLP $\to$ CE on train labels \\
Optimizer & AdamW \\
Learning rate & $1\times10^{-4}$ \\
Weight decay & $1\times10^{-4}$ \\
Epochs $\times$ steps/epoch & $10 \times 100 = 1{,}000$ meta-steps \\
Batch size & full train split \\
\bottomrule
\end{tabular}
\end{table}
\paragraph{Evaluation.} We largely follow GraphAny~\cite{zhao2025graphany} in benchmark selection, as detailed in \S\ref{sec:nc}. We use official splits for every dataset where available, otherwise we sample 20 samples per class as training data, then split the reminder evenly for validation and test data. We note that each of the feature views $\{ \X^{(k)}\}_{k=1}^K$ can be transformed using $\Theta^{(T)}$, hence we fit an MLP on the train set for each $k$ and use the validation set to select the best view to evaluate on the test.

\section{Detailed Main Results}\label{sec:main_results}

\subsection{Link Prediction Results}
\begin{table*}[h!]
  \caption{Link prediction results across attributed and non-attributed
  benchmarks. \SIGILLP, UniLP and TFMLinker are fully-inductive and evaluated
  zero-shot on every target graph; all other methods are transductive and
  trained on each target graph. Best in \textbf{bold}, second-best
  \underline{underlined}. $\geq$24h denotes methods where per-query inference
  exceeds 24 hours; ``--'' marks results which are not producible due to lack of published code.}
  \label{tab:lp_main}
  \centering
  \footnotesize
  \setlength{\tabcolsep}{1.5pt}
  \begin{tabular}{@{}lccccccccccc@{}}
    \toprule
    & \multicolumn{7}{c}{\textbf{Attributed}} & \multicolumn{4}{c}{\textbf{Non-attributed}} \\
    \cmidrule(lr){2-8} \cmidrule(lr){9-12}
    Dataset & Cora & CiteSeer & PubMed & AmzComp & AmzPhoto & CoCS & CoPhys & C.ele & USAir & PB & NS \\
    Metric & \emph{H@100} & \emph{H@100} & \emph{H@100} & \emph{H@50} & \emph{H@50} & \emph{H@50} & \emph{H@50} & \emph{H@50} & \emph{H@50} & \emph{H@50} & \emph{H@50} \\
    \midrule
    \multicolumn{12}{c}{\textbf{Transductive}} \\
    \midrule
    NBFNet    & 71.65\scriptsize{$\pm$}2.27 & 74.07\scriptsize{$\pm$}1.75 & 58.73\scriptsize{$\pm$}1.99 & 13.14\scriptsize{$\pm$}1.04 & 21.35\scriptsize{$\pm$}3.36 & 71.86\scriptsize{$\pm$}0.93 & 66.71\scriptsize{$\pm$}4.39 & 38.00\scriptsize{$\pm$}4.92 & 52.71\scriptsize{$\pm$}5.27 & 7.17\scriptsize{$\pm$}0.08 & \underline{89.60\scriptsize{$\pm$}1.83} \\
    MPLP+     & 79.93\scriptsize{$\pm$}1.35 & 83.96\scriptsize{$\pm$}1.89 & \underline{81.27\scriptsize{$\pm$}1.35} & \textbf{42.21\scriptsize{$\pm$}3.56} & \textbf{57.76\scriptsize{$\pm$}2.75} & \underline{75.55\scriptsize{$\pm$}1.46} & \textbf{76.36\scriptsize{$\pm$}1.40} & 62.56\scriptsize{$\pm$}4.79 & 85.08\scriptsize{$\pm$}1.54 & \underline{48.01\scriptsize{$\pm$}2.94} & 80.16\scriptsize{$\pm$}1.18 \\
    NCNC      & \textbf{89.65\scriptsize{$\pm$}1.36} & \textbf{93.47\scriptsize{$\pm$}0.95} & \textbf{81.29\scriptsize{$\pm$}0.95} & \underline{36.48\scriptsize{$\pm$}4.16} & \underline{47.98\scriptsize{$\pm$}2.36} & 74.65\scriptsize{$\pm$}1.23 & \underline{75.96\scriptsize{$\pm$}1.73} & 64.45\scriptsize{$\pm$}5.10 & 85.85\scriptsize{$\pm$}2.46 & 47.75\scriptsize{$\pm$}6.90 & 88.25\scriptsize{$\pm$}2.25 \\
    ELPH      & \underline{87.72\scriptsize{$\pm$}2.13} & \underline{93.44\scriptsize{$\pm$}0.53} & 72.99\scriptsize{$\pm$}1.43 & 29.01\scriptsize{$\pm$}2.66 & 43.51\scriptsize{$\pm$}2.37 & 72.26\scriptsize{$\pm$}2.58 & 65.80\scriptsize{$\pm$}2.26 & 60.51\scriptsize{$\pm$}5.72 & 84.52\scriptsize{$\pm$}2.08 & 43.58\scriptsize{$\pm$}3.48 & 86.08\scriptsize{$\pm$}0.69 \\
    \midrule
    \multicolumn{12}{c}{\textbf{Fully-inductive}} \\
    \midrule
    UniLP     & 74.01\scriptsize{$\pm$}2.58 & 67.12\scriptsize{$\pm$}1.94 & 59.09\scriptsize{$\pm$}0.87 & $\geq$ 24h & 34.71\scriptsize{$\pm$}1.48 & 69.38\scriptsize{$\pm$}1.70 & $\geq$ 24h & 65.20\scriptsize{$\pm$}4.40 & 85.98\scriptsize{$\pm$}2.00 & \textbf{48.14\scriptsize{$\pm$}2.99} & 89.09\scriptsize{$\pm$}2.05 \\
    TFMLinker & -- & -- & -- & -- & -- & -- & -- & \textbf{68.62\scriptsize{$\pm$}5.97} & \textbf{91.67\scriptsize{$\pm$}1.37} & 47.75\scriptsize{$\pm$}2.76 & \textbf{90.46\scriptsize{$\pm$}2.43} \\
    \SIGILLP & 77.99\scriptsize{$\pm$}1.72$^{\dagger}$ & 72.95\scriptsize{$\pm$}1.65 & 60.01\scriptsize{$\pm$}1.02 & 26.25\scriptsize{$\pm$}1.70 & 37.15\scriptsize{$\pm$}1.55 & \textbf{75.96\scriptsize{$\pm$}1.18} & 70.28\scriptsize{$\pm$}2.54 & \underline{68.22\scriptsize{$\pm$}2.41} & 88.39\scriptsize{$\pm$}0.67 & 45.55\scriptsize{$\pm$}3.72 & 88.69\scriptsize{$\pm$}1.30 \\
    \bottomrule
  \end{tabular}
\end{table*}
\newpage
\subsection{Knowledge Graph Reasoning Results}
\begin{table*}[h!]
  \caption{Zero-shot KG reasoning results on the 23 inductive $(e,r)$ graphs from
  ULTRA's benchmark (unseen entities and relations). Sigil$(K)$ denotes
  \SIGILLP with SIG order $K$. ULTRA
  numbers are the zero-shot results reported in Table~10 and Table~11 of~\cite{galkin2023ultra}. Best in \textbf{bold}, second-best \underline{underlined}.}
  \label{tab:kg_er}
  \centering
  \scriptsize
  \setlength{\tabcolsep}{3pt}
  \begin{tabular}{@{}lcccccccc@{}}
    \toprule
    & \multicolumn{2}{c}{ULTRA} & \multicolumn{2}{c}{\SIGIL$(0)$}
    & \multicolumn{2}{c}{\SIGIL$(1)$} & \multicolumn{2}{c}{\SIGIL$(2)$} \\
    \cmidrule(lr){2-3}\cmidrule(lr){4-5}\cmidrule(lr){6-7}\cmidrule(lr){8-9}
    Dataset & MRR & H@10 & MRR & H@10 & MRR & H@10 & MRR & H@10 \\
    \midrule
    FB-25       & \textbf{38.8} & \textbf{64.0} & \underline{35.0} & 60.7 & 34.3 & 60.2 & 34.8 & \underline{61.1} \\
    FB-50       & \textbf{33.8} & \textbf{54.3} & 28.7 & 51.0 & 29.3 & 50.4 & \underline{29.3} & \underline{51.3} \\
    FB-75       & \textbf{40.3} & \textbf{60.4} & 33.6 & 55.3 & \underline{34.2} & 56.1 & 33.8 & \underline{56.5} \\
    FB-100      & \textbf{44.9} & \textbf{64.2} & 34.5 & 56.6 & \underline{35.2} & 57.2 & 35.0 & \underline{57.9} \\
    WK-25       & \textbf{31.6} & \textbf{53.2} & 22.5 & 43.5 & 22.4 & 40.2 & \underline{23.7} & \underline{44.3} \\
    WK-50       & \textbf{16.6} & \textbf{32.4} & 12.9 & \underline{27.4} & \underline{13.1} & 25.9 & 12.1 & 26.5 \\
    WK-75       & \textbf{36.5} & \textbf{53.7} & 31.1 & 47.1 & \underline{32.3} & 46.7 & 31.6 & \underline{48.1} \\
    WK-100      & 16.4 & \textbf{28.6} & 16.3 & 27.8 & \textbf{16.9} & 27.3 & \underline{16.9} & \underline{28.0} \\
    NL-0        & \textbf{34.2} & \textbf{52.3} & \underline{30.3} & \underline{48.4} & 28.4 & 43.8 & 28.4 & 44.4 \\
    NL-25       & \textbf{39.5} & \textbf{56.9} & 27.7 & 41.3 & 26.6 & 38.7 & \underline{30.5} & \underline{43.7} \\
    NL-50       & \textbf{40.7} & \textbf{57.0} & \underline{32.2} & \underline{47.3} & 27.7 & 42.9 & 31.2 & 45.6 \\
    NL-75       & \textbf{36.8} & \textbf{54.7} & \underline{22.6} & \underline{40.4} & 20.3 & 35.9 & 22.2 & 39.4 \\
    NL-100      & \textbf{47.1} & \textbf{65.1} & 34.9 & \underline{53.8} & 34.4 & 53.3 & \underline{35.1} & 52.3 \\
    Metafam     & 23.8 & 64.4 & \underline{30.9} & \underline{84.0} & \textbf{34.6} & \textbf{87.8} & 26.3 & 74.2 \\
    FBNELL      & \textbf{48.5} & \textbf{65.2} & \underline{45.5} & \underline{63.3} & 45.4 & 62.6 & 39.6 & 60.7 \\
    MT1 tax     & 22.4 & 30.5 & \textbf{41.2} & \underline{50.2} & 39.0 & 50.1 & \underline{40.6} & \textbf{52.3} \\
    MT1 health  & 29.8 & 37.4 & \textbf{37.9} & \underline{45.8} & \underline{37.8} & \textbf{46.4} & 37.5 & 45.1 \\
    MT2 org     & \textbf{9.5}  & \textbf{15.9} & 8.9 & 14.7 & 8.6 & 13.9 & \underline{9.2} & \underline{15.8} \\
    MT2 sci     & 25.8 & 35.4 & \underline{37.5} & \underline{53.1} & 37.1 & 51.9 & \textbf{37.9} & \textbf{53.7} \\
    MT3 art     & 25.9 & 40.2 & \textbf{28.5} & \underline{43.2} & 27.9 & 42.5 & \underline{28.1} & \textbf{44.4} \\
    MT3 infra   & 61.9 & 75.5 & \underline{63.3} & \underline{78.2} & 63.2 & 77.9 & \textbf{64.2} & \textbf{78.4} \\
    MT4 sci     & 27.4 & \underline{44.9} & 27.8 & 42.6 & \underline{28.2} & 42.8 & \textbf{29.2} & \textbf{45.5} \\
    MT4 health  & 62.4 & 73.7 & \underline{66.5} & 76.2 & \textbf{67.6} & \underline{77.1} & 66.2 & \textbf{77.3} \\
    \midrule
    \emph{Average} & \textbf{34.5} & \textbf{51.3} & \underline{32.6} & \underline{50.1} & 32.4 & 49.2 & 32.3 & 49.8 \\
    \bottomrule
  \end{tabular}
\end{table*}

\begin{table*}[h!]
  \caption{Zero-shot KG reasoning results on the 23 inductive $(e)$ graphs from
  ULTRA's benchmark (unseen entities and relations). Sigil$(K)$ denotes
  \SIGILLP with SIG order $K$. ULTRA
  numbers are the zero-shot results reported in Table~10 and Table~11 of~\cite{galkin2023ultra}. Best in \textbf{bold}, second-best \underline{underlined}.}
  \label{tab:kg_e}
  \centering
  \scriptsize
  \setlength{\tabcolsep}{3pt}
  \begin{tabular}{@{}lcccccccc@{}}
    \toprule
    & \multicolumn{2}{c}{ULTRA} & \multicolumn{2}{c}{\SIGIL$(0)$}
    & \multicolumn{2}{c}{\SIGIL$(1)$} & \multicolumn{2}{c}{\SIGIL$(2)$} \\
    \cmidrule(lr){2-3}\cmidrule(lr){4-5}\cmidrule(lr){6-7}\cmidrule(lr){8-9}
    Dataset & MRR & H@10 & MRR & H@10 & MRR & H@10 & MRR & H@10 \\
    \midrule
    FB v1       & \underline{49.8} & 65.6 & 49.6 & \textbf{67.9} & 48.9 & 67.2 & \textbf{49.8} & \underline{67.8} \\
    FB v2       & 51.2 & 70.0 & \underline{53.1} & \underline{71.8} & 52.0 & 71.4 & \textbf{53.3} & \textbf{73.2} \\
    FB v3       & 49.1 & 65.4 & \underline{50.0} & \underline{66.4} & 48.5 & 65.4 & \textbf{50.2} & \textbf{66.9} \\
    FB v4       & 48.6 & \underline{67.7} & \underline{48.6} & 67.7 & 47.6 & 66.6 & \textbf{48.9} & \textbf{68.0} \\
    WN v1       & 64.8 & 76.8 & 18.9 & 38.9 & \underline{69.1} & \underline{78.4} & \textbf{69.4} & \textbf{79.4} \\
    WN v2       & 66.3 & 76.5 & 67.6 & \textbf{79.2} & \underline{68.2} & 78.0 & \textbf{69.1} & \underline{78.8} \\
    WN v3       & 37.6 & 47.6 & \textbf{46.1} & \underline{59.3} & \underline{46.0} & 58.9 & 45.5 & \textbf{59.8} \\
    WN v4       & 61.1 & 70.5 & 41.9 & 72.2 & \underline{64.1} & \underline{72.5} & \textbf{64.3} & \textbf{72.6} \\
    NELL v1     & \textbf{78.5} & \textbf{91.3} & 64.2 & 71.4 & 64.6 & 80.6 & \underline{68.6} & \underline{90.8} \\
    NELL v2     & 52.6 & 70.7 & \textbf{53.5} & \textbf{74.7} & \underline{53.1} & \underline{73.2} & 51.3 & 72.2 \\
    NELL v3     & \textbf{51.5} & 70.2 & 50.9 & \textbf{71.9} & \underline{51.3} & \underline{70.2} & 49.1 & 68.5 \\
    NELL v4     & 47.9 & 71.2 & 50.3 & 72.1 & \textbf{52.2} & \textbf{73.2} & \underline{50.6} & \underline{72.6} \\
    ILPC Small  & \textbf{30.2} & \textbf{44.3} & \underline{28.7} & 43.0 & 27.8 & 41.6 & 28.1 & \underline{43.2} \\
    ILPC Large  & \underline{29.0} & \textbf{42.4} & \textbf{29.8} & \underline{41.8} & 27.6 & 39.3 & 28.0 & 41.2 \\
    HM 1k       & 5.9  & 9.2  & \underline{6.7} & \underline{12.5} & 6.1 & 10.3 & \textbf{6.8} & \textbf{13.1} \\
    HM 3k       & 3.7  & 7.7  & \underline{6.4} & \underline{10.6} & 5.7 & 10.2 & \textbf{6.4} & \textbf{11.9} \\
    HM 5k       & 3.4  & 7.1  & \underline{5.9} & \underline{9.7}  & 5.1 & 8.5  & \textbf{6.0} & \textbf{10.8} \\
    HM indigo   & \textbf{44.0} & \textbf{64.8} & \underline{38.8} & \underline{59.3} & 37.3 & 57.7 & 37.8 & 58.3 \\
    \midrule
    \emph{Average} & 43.1 & 56.6 & 39.5 & 55.0 & \underline{43.1} & \underline{56.8} & \textbf{43.5} & \textbf{58.3} \\
    \bottomrule
  \end{tabular}
\end{table*}
 
\newpage
\subsection{Node Classification Results}
\begin{table*}[h!]
  \caption{Node classification accuracy (\%) on 26 benchmarks against
  \textbf{transductive} and \textbf{fully-inductive} baselines. Best in
  \textbf{bold}, second-best \underline{underlined}. $^{\dagger}$RGVT is
  pretrained on ogbn-arxiv. GraphAny, TS-MEAN and \SIGILNC are pretrained on
  Cora. TAG is pretrained on synthetic datasets. Evaluations for methods other
  than \SIGILNC are taken from~\cite{lee2026view},
  including their updated transductive baselines.}
  \label{tab:main_nc}
  \centering
  \footnotesize
  \begin{tabular}{lcccccccc}
    \toprule
    & \multicolumn{3}{c}{\textbf{Transductive}} & \multicolumn{4}{c}{\textbf{Fully-Inductive}} & \multicolumn{1}{c}{\textbf{Ours}} \\
    \cmidrule(lr){2-4} \cmidrule(lr){5-8} \cmidrule(lr){9-9}
    Dataset & MLP & GCN & GAT & GraphAny & TS-MEAN & TAG & RGVT & \SIGILNC \\
    \midrule
    Actor & \underline{33.95\scriptsize{$\pm$}0.80} & 30.37\scriptsize{$\pm$}0.44 & 28.57\scriptsize{$\pm$}0.82 & 27.91\scriptsize{$\pm$}0.16 & 28.09\scriptsize{$\pm$}0.93 & 31.18\scriptsize{$\pm$}0.18 & 32.68\scriptsize{$\pm$}1.26 & \textbf{36.49\scriptsize{$\pm$}0.50} \\
    AirBrazil & 23.08\scriptsize{$\pm$}5.83 & 63.08\scriptsize{$\pm$}3.44 & 44.62\scriptsize{$\pm$}14.54 & 33.07\scriptsize{$\pm$}16.68 & 39.23\scriptsize{$\pm$}5.70 & \underline{73.08\scriptsize{$\pm$}4.03} & \textbf{75.38\scriptsize{$\pm$}6.44} & 39.23\scriptsize{$\pm$}5.01 \\
    AirEurope & 21.25\scriptsize{$\pm$}2.31 & 41.50\scriptsize{$\pm$}0.84 & 40.00\scriptsize{$\pm$}1.93 & 40.50\scriptsize{$\pm$}7.01 & 35.88\scriptsize{$\pm$}6.91 & \underline{55.38\scriptsize{$\pm$}2.30} & \textbf{57.12\scriptsize{$\pm$}3.55} & 38.50\scriptsize{$\pm$}6.38 \\
    AirUSA & 22.88\scriptsize{$\pm$}1.46 & 53.30\scriptsize{$\pm$}1.18 & 51.24\scriptsize{$\pm$}2.90 & 43.46\scriptsize{$\pm$}1.45 & 42.34\scriptsize{$\pm$}2.12 & \textbf{60.50\scriptsize{$\pm$}0.80} & \underline{57.55\scriptsize{$\pm$}1.28} & 48.25\scriptsize{$\pm$}3.35 \\
    AmzComp & 58.28\scriptsize{$\pm$}2.98 & \textbf{84.69\scriptsize{$\pm$}0.18} & \underline{84.43\scriptsize{$\pm$}1.28} & 82.99\scriptsize{$\pm$}1.22 & 81.37\scriptsize{$\pm$}1.25 & 84.33\scriptsize{$\pm$}0.33 & 84.14\scriptsize{$\pm$}1.76 & 72.84\scriptsize{$\pm$}2.65 \\
    AmzPhoto & 68.20\scriptsize{$\pm$}0.88 & \underline{91.71\scriptsize{$\pm$}0.18} & 89.40\scriptsize{$\pm$}0.71 & 90.14\scriptsize{$\pm$}0.93 & 90.18\scriptsize{$\pm$}1.30 & 89.63\scriptsize{$\pm$}0.48 & \textbf{92.29\scriptsize{$\pm$}0.44} & 84.13\scriptsize{$\pm$}1.76 \\
    AmzRatings & 47.90\scriptsize{$\pm$}0.45 & \textbf{49.17\scriptsize{$\pm$}0.43} & \underline{48.99\scriptsize{$\pm$}0.56} & 42.84\scriptsize{$\pm$}0.04 & 42.27\scriptsize{$\pm$}1.40 & 44.34\scriptsize{$\pm$}0.32 & 45.58\scriptsize{$\pm$}1.16 & 48.49\scriptsize{$\pm$}2.39 \\
    BlogCatalog & 64.11\scriptsize{$\pm$}1.95 & \textbf{86.00\scriptsize{$\pm$}0.52} & 55.16\scriptsize{$\pm$}4.77 & 72.52\scriptsize{$\pm$}3.22 & 76.30\scriptsize{$\pm$}2.92 & 79.77\scriptsize{$\pm$}1.06 & \underline{81.39\scriptsize{$\pm$}2.75} & 68.77\scriptsize{$\pm$}1.07 \\
    Chameleon & 36.62\scriptsize{$\pm$}0.87 & 41.31\scriptsize{$\pm$}3.05 & 39.83\scriptsize{$\pm$}2.10 & 37.99\scriptsize{$\pm$}4.54 & 43.83\scriptsize{$\pm$}3.99 & \textbf{44.98\scriptsize{$\pm$}1.37} & \underline{44.74\scriptsize{$\pm$}2.17} & 42.84\scriptsize{$\pm$}5.01 \\
    Citeseer & 44.40\scriptsize{$\pm$}0.44 & 69.22\scriptsize{$\pm$}0.62 & \underline{69.46\scriptsize{$\pm$}1.36} & 68.90\scriptsize{$\pm$}0.07 & 68.66\scriptsize{$\pm$}0.19 & 68.08\scriptsize{$\pm$}0.61 & \textbf{70.12\scriptsize{$\pm$}1.47} & 53.70\scriptsize{$\pm$}2.16 \\
    CoCS & 85.88\scriptsize{$\pm$}0.93 & 90.60\scriptsize{$\pm$}0.09 & 90.27\scriptsize{$\pm$}0.57 & 90.47\scriptsize{$\pm$}0.63 & \underline{90.92\scriptsize{$\pm$}0.47} & 90.49\scriptsize{$\pm$}0.50 & \textbf{92.28\scriptsize{$\pm$}0.37} & 90.12\scriptsize{$\pm$}0.41 \\
    CoPhysics & 87.43\scriptsize{$\pm$}1.98 & \underline{92.85\scriptsize{$\pm$}0.24} & 91.22\scriptsize{$\pm$}2.06 & 92.70\scriptsize{$\pm$}0.54 & 92.61\scriptsize{$\pm$}0.61 & 92.31\scriptsize{$\pm$}0.27 & \textbf{93.18\scriptsize{$\pm$}0.51} & 92.76\scriptsize{$\pm$}0.43 \\
    Cornell & 67.57\scriptsize{$\pm$}5.06 & 45.41\scriptsize{$\pm$}3.52 & 40.00\scriptsize{$\pm$}2.26 & 64.86\scriptsize{$\pm$}1.91 & 68.65\scriptsize{$\pm$}2.42 & \textbf{75.14\scriptsize{$\pm$}2.16} & \underline{74.05\scriptsize{$\pm$}2.42} & 60.54\scriptsize{$\pm$}4.52 \\
    DBLP & 56.27\scriptsize{$\pm$}0.62 & \textbf{79.52\scriptsize{$\pm$}1.03} & \underline{79.07\scriptsize{$\pm$}0.78} & 71.73\scriptsize{$\pm$}0.94 & 66.42\scriptsize{$\pm$}3.65 & 71.92\scriptsize{$\pm$}1.46 & 78.12\scriptsize{$\pm$}1.19 & 75.12\scriptsize{$\pm$}0.55 \\
    FullCora & 33.54\scriptsize{$\pm$}0.64 & \underline{62.69\scriptsize{$\pm$}0.13} & 60.71\scriptsize{$\pm$}0.16 & 56.73\scriptsize{$\pm$}0.41 & 53.58\scriptsize{$\pm$}0.73 & 54.56\scriptsize{$\pm$}0.19 & \textbf{63.26\scriptsize{$\pm$}0.75} & 57.11\scriptsize{$\pm$}0.36 \\
    Minesweeper & 80.00\scriptsize{$\pm$}0.00 & 80.16\scriptsize{$\pm$}0.26 & 81.70\scriptsize{$\pm$}0.53 & 80.46\scriptsize{$\pm$}0.15 & 80.68\scriptsize{$\pm$}0.38 & \textbf{85.83\scriptsize{$\pm$}0.13} & 80.00\scriptsize{$\pm$}0.00 & \underline{85.02\scriptsize{$\pm$}0.70} \\
    ogbn-arxiv & 55.50\scriptsize{$\pm$}0.23 & 71.19\scriptsize{$\pm$}0.30 & \underline{71.89\scriptsize{$\pm$}0.24} & 58.62\scriptsize{$\pm$}0.05 & 56.33\scriptsize{$\pm$}2.58 & 66.70\scriptsize{$\pm$}0.21 & 70.33\scriptsize{$\pm$}0.35$^{\dagger}$ & \textbf{71.92\scriptsize{$\pm$}0.05} \\
    Pubmed & 69.50\scriptsize{$\pm$}1.79 & 76.90\scriptsize{$\pm$}0.42 & 76.26\scriptsize{$\pm$}0.70 & 76.60\scriptsize{$\pm$}0.31 & 74.98\scriptsize{$\pm$}0.56 & \textbf{78.96\scriptsize{$\pm$}0.43} & \underline{78.20\scriptsize{$\pm$}0.82} & 77.64\scriptsize{$\pm$}1.12 \\
    Questions & \textbf{97.33\scriptsize{$\pm$}0.06} & 97.06\scriptsize{$\pm$}0.02 & 97.06\scriptsize{$\pm$}0.02 & 97.06\scriptsize{$\pm$}0.03 & 97.02\scriptsize{$\pm$}0.01 & \underline{97.14\scriptsize{$\pm$}0.01} & 97.13\scriptsize{$\pm$}0.06 & 97.12\scriptsize{$\pm$}0.05 \\
    Roman & 65.80\scriptsize{$\pm$}0.35 & 45.33\scriptsize{$\pm$}0.25 & 43.42\scriptsize{$\pm$}0.18 & 64.25\scriptsize{$\pm$}0.64 & 66.36\scriptsize{$\pm$}1.02 & \underline{74.12\scriptsize{$\pm$}0.28} & 70.53\scriptsize{$\pm$}0.70 & \textbf{76.80\scriptsize{$\pm$}0.26} \\
    Squirrel & 30.36\scriptsize{$\pm$}0.78 & 38.67\scriptsize{$\pm$}1.84 & 38.78\scriptsize{$\pm$}2.39 & 37.25\scriptsize{$\pm$}2.65 & 39.80\scriptsize{$\pm$}0.53 & \textbf{46.29\scriptsize{$\pm$}1.01} & \underline{41.77\scriptsize{$\pm$}2.06} & 38.06\scriptsize{$\pm$}2.14 \\
    Texas & 48.65\scriptsize{$\pm$}4.01 & 51.35\scriptsize{$\pm$}2.71 & 54.05\scriptsize{$\pm$}2.41 & 71.89\scriptsize{$\pm$}1.48 & 73.51\scriptsize{$\pm$}4.01 & \textbf{81.08\scriptsize{$\pm$}2.09} & \underline{78.92\scriptsize{$\pm$}2.26} & 62.70\scriptsize{$\pm$}6.45 \\
    Tolokers & 78.16\scriptsize{$\pm$}0.02 & 80.60\scriptsize{$\pm$}0.12 & \underline{81.67\scriptsize{$\pm$}0.18} & 78.20\scriptsize{$\pm$}0.02 & 78.12\scriptsize{$\pm$}0.09 & \textbf{82.82\scriptsize{$\pm$}0.15} & 78.32\scriptsize{$\pm$}0.35 & 80.71\scriptsize{$\pm$}0.73 \\
    Wiki & 63.79\scriptsize{$\pm$}1.77 & 69.24\scriptsize{$\pm$}0.76 & 58.64\scriptsize{$\pm$}1.93 & 60.56\scriptsize{$\pm$}3.62 & 69.89\scriptsize{$\pm$}1.31 & \underline{70.09\scriptsize{$\pm$}0.92} & \textbf{72.26\scriptsize{$\pm$}4.66} & 68.38\scriptsize{$\pm$}2.02 \\
    WikiCS & 72.72\scriptsize{$\pm$}0.43 & 79.07\scriptsize{$\pm$}0.24 & \underline{79.44\scriptsize{$\pm$}0.40} & 74.39\scriptsize{$\pm$}0.71 & 74.16\scriptsize{$\pm$}2.07 & 79.07\scriptsize{$\pm$}0.54 & 78.92\scriptsize{$\pm$}0.58 & \textbf{79.49\scriptsize{$\pm$}0.37} \\
    Wisconsin & 66.67\scriptsize{$\pm$}3.51 & 37.25\scriptsize{$\pm$}1.64 & 52.94\scriptsize{$\pm$}3.10 & 61.18\scriptsize{$\pm$}5.08 & 61.18\scriptsize{$\pm$}11.38 & \textbf{83.14\scriptsize{$\pm$}1.33} & \underline{78.04\scriptsize{$\pm$}4.47} & 71.76\scriptsize{$\pm$}8.27 \\
    \midrule
    \textbf{Average} & 56.91 & 65.70 & 63.42 & 64.51 & 65.09 & \underline{71.57} & \textbf{71.78} & 66.10 \\
    \bottomrule
  \end{tabular}
\end{table*}
\newpage
\section{Additional Experiments}\label{sec:more_experiments}

\subsection{Ablations and Parameter Experiments}
\paragraph{Setup.} We set to examine the impact of different parameter choices and design choices on the performance of \SIGILLP. We conduct experiments on the order of the SIG $K$ (the number of hops in construction), the choice of the interaction function $g$ (between Hadamard product and absolute difference), and the number of SIG and Downstream NBFNet layers.

\paragraph{Results.} Figure~\ref{fig:g_hits50} presents results for selecting between \textsc{Absdiff} and \textsc{Hadamard} interaction operators. We can note that in general, one is not clearly surpassing the other, suggesting that both manages to carry good signals for downstream SIG construction that also might be overlapping. Figure~\ref{fig:hops_hits50} presents results for various SIG orders. Here we can clearly see that between different datasets, there are different optimal choices - which is an expected results and a key limitation for the GFM settings in advance - different structure requires different receptive fields for the representation learning, which is a product of the SIG order for \SIGIL. Figures~\ref{fig:down_hits50} and \ref{fig:sig_hits50} present results for different number of SIG and downstream layers in \SIGILLP. while there is a general improvement with more layers, which creates a more parameterized model, some datasets does not benefit from more layers and performance can degrade, which ties into the insight from the hops experiment - a priori, different GFMs with different structural properties and receptive fields might behave differently in downstream tasks. 

\begin{figure}[tbph]
    \centering
    \includegraphics[width=0.75\linewidth]{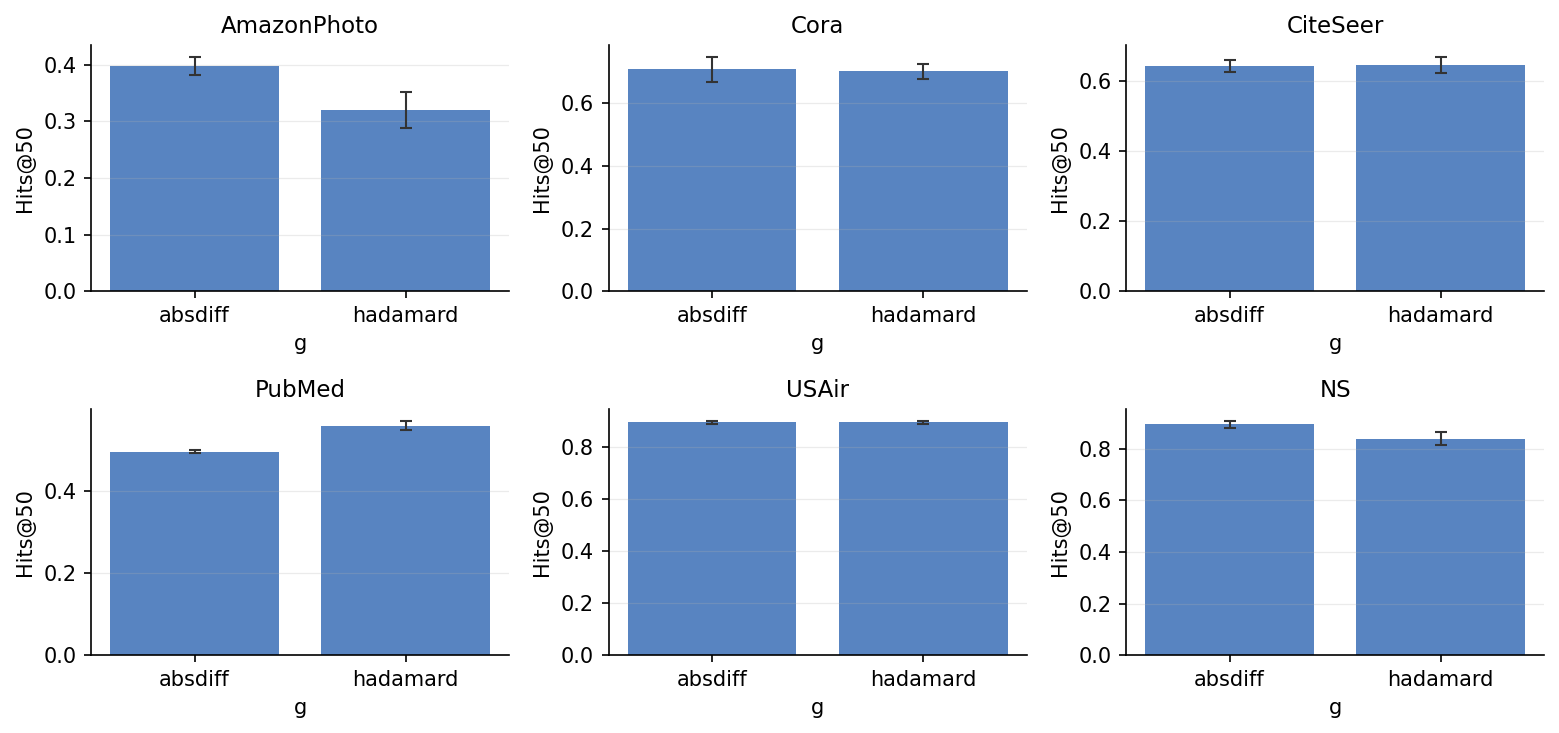}
    \caption{Interaction operator $g$ vs. Hits@50}
    \label{fig:g_hits50}
\end{figure}




\begin{figure}[htbp]
    \centering
    \begin{subfigure}[b]{0.48\textwidth}
        \centering
        \includegraphics[width=\textwidth]{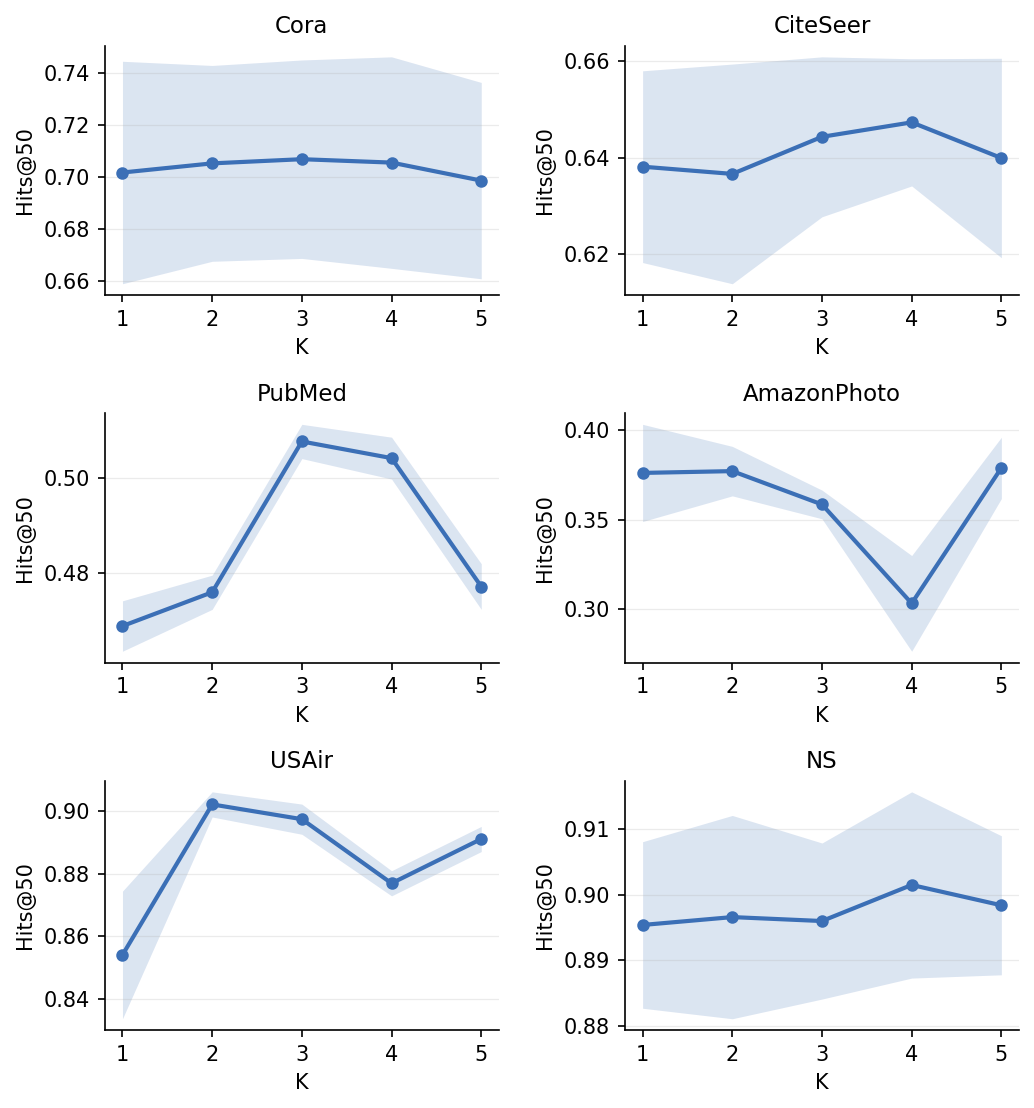}
        \caption{Number of hops $k$ vs. Hits@50}
        \label{fig:hops_hits50}
    \end{subfigure}
    \hfill 
    \begin{subfigure}[b]{0.48\textwidth}
        \centering
        \includegraphics[width=\textwidth]{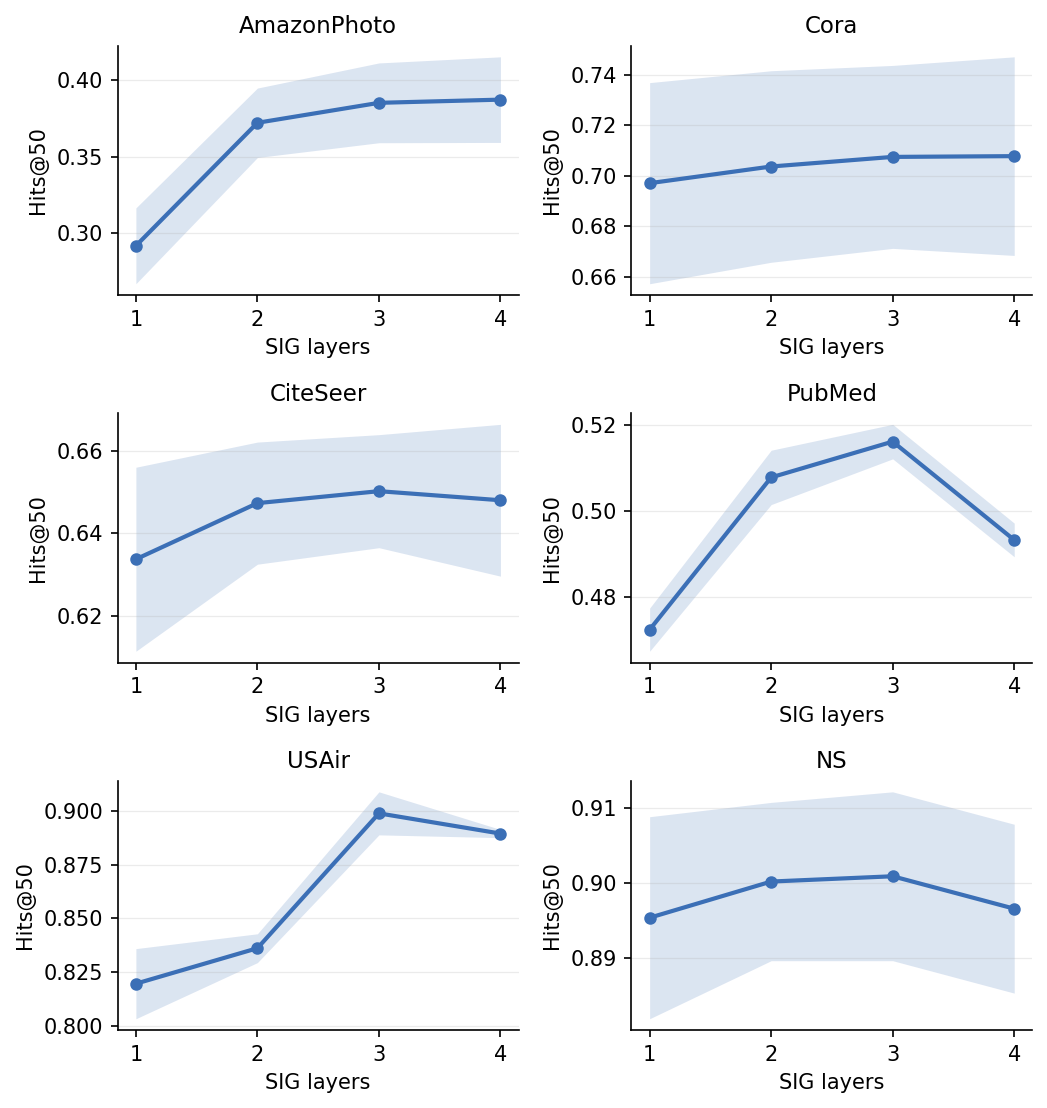}
        \caption{Number of SIG NBNet layers vs. Hits@50}
        \label{fig:sig_hits50}
    \end{subfigure}

    \begin{subfigure}[b]{0.48\textwidth}
        \centering
        \includegraphics[width=\textwidth]{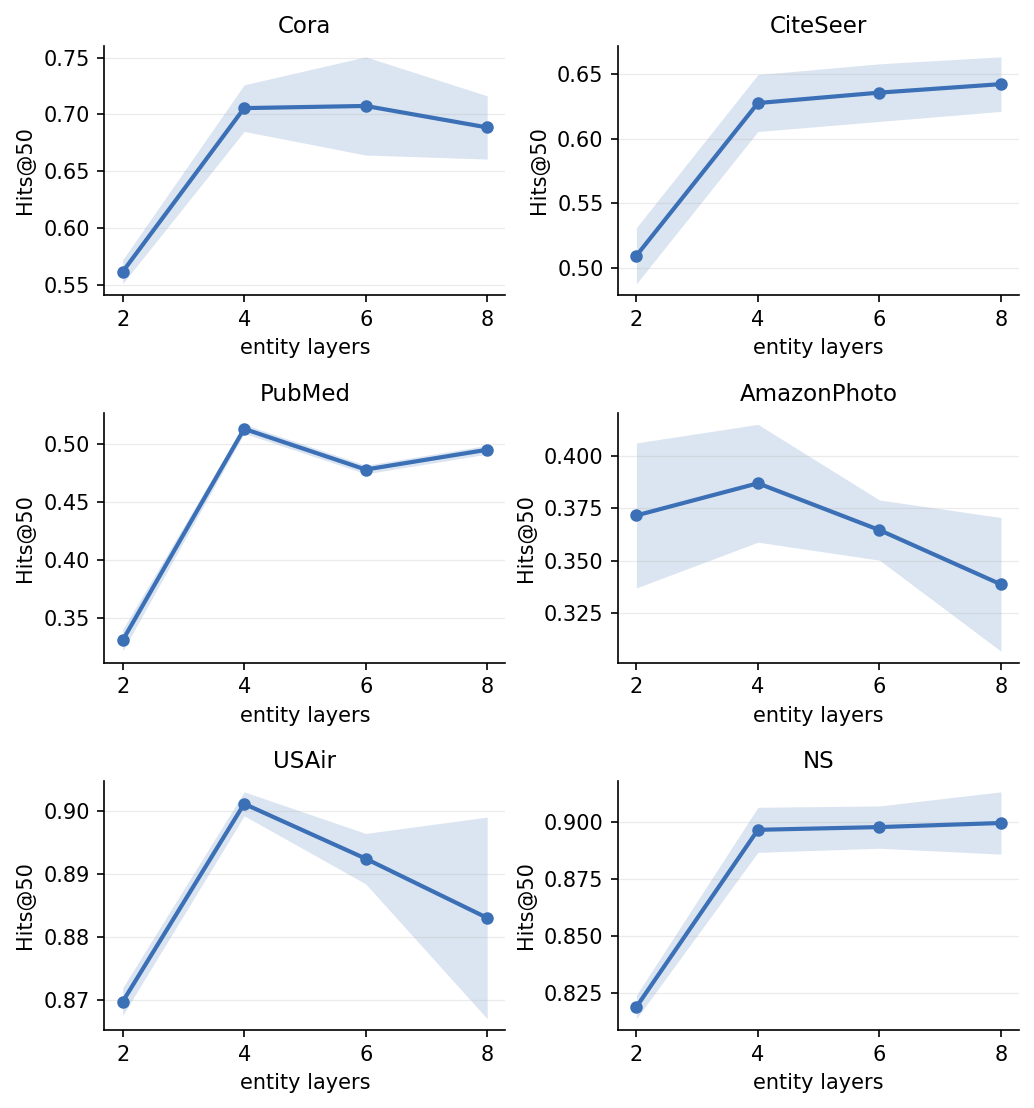}
        \caption{Number of Downstream NBFNet layers vs. Hits@50}
        \label{fig:down_hits50}
    \end{subfigure}
    \label{fig:main_figure}
\end{figure}

\subsection{Trainsize for Downstream MLP in \SIGILNC}
As discussed in~\ref{sec:experiments}, \SIGIL effectivly compresses node features of various dimensions to a fixed-size space, and we hypothesized that it is a possible cause for \SIGILNC relative weakness in NC against methods that do not compress the features. We ask whether providing the downstream MLP that adapts the fixed-size representations to the downstream label space more supervision will cause improvements. To answer this question, we conduct the following experiment. We take the pretrained checkpoint on Cora, and for each dataset in the NC evaluation protocol we train multiple MLPs with growing sizes of supervision, going from 1 train sample per class to $\min(500, \frac{\text{\#train+val}}{C})$ and inference on test. Results are presented in figure~\ref{fig:trainsize}.

\begin{figure}[tbph]
    \centering
    \includegraphics[width=0.95\linewidth]{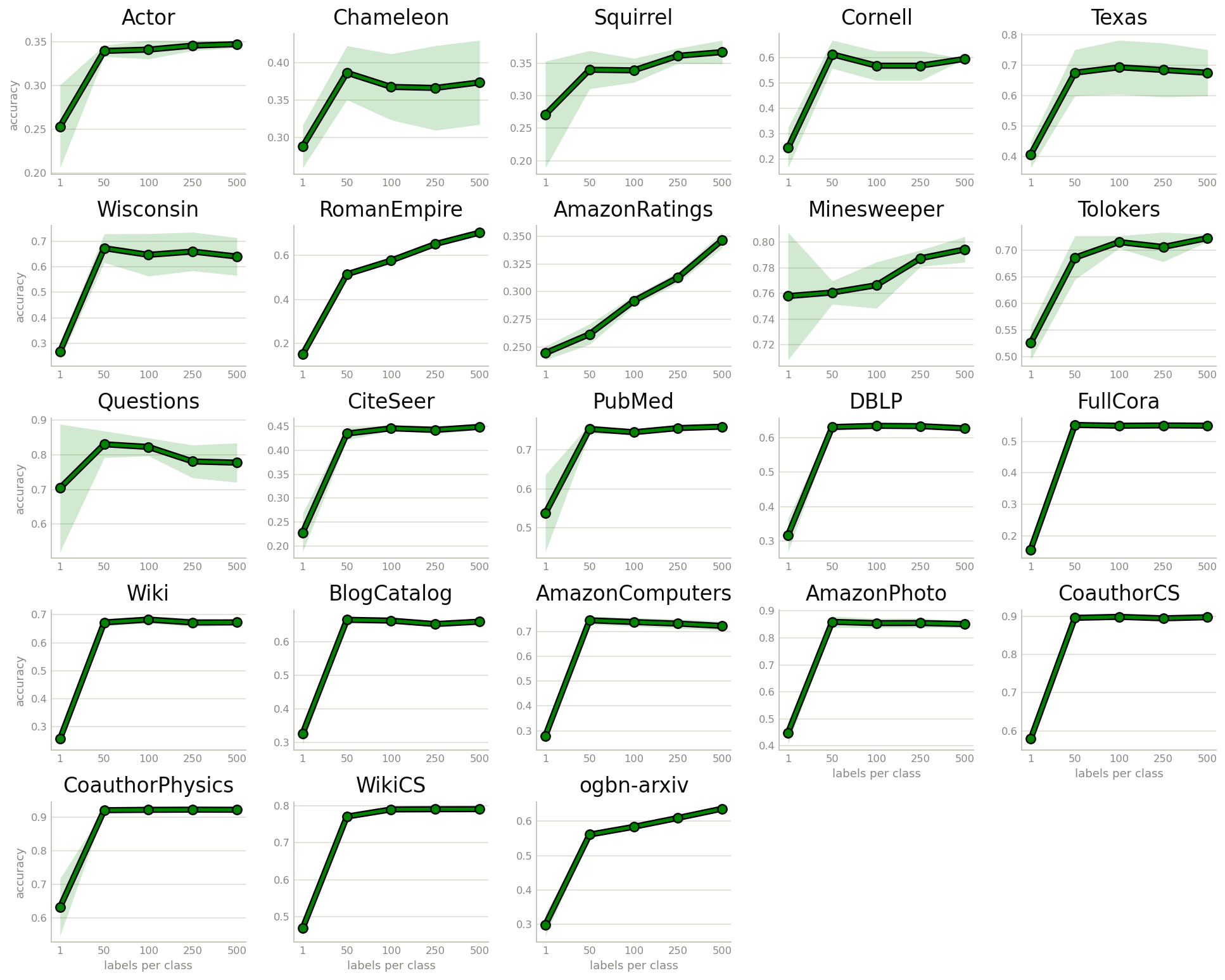}
    \caption{Train size supervision for downstream MLP}
    \label{fig:trainsize}
\end{figure}

\subsection{Visualizations of \SIGILNC Learned Space}
Figure~\ref{fig:embed_viz} presents t-SNE plots of the embeddings space learned by \SIGILNC, against embeddings from a transductive GAT~\cite{velickovic2018graph}. We can see that \SIGILNC manages to outputs meaningful representation which are coherent for the classification tasks even without training of the downstream graphs, and even shows robustness in the embeddings towards more heterophilic graphs, which the GAT has a bias against (based on the known smoothing property of plain message passing).
\begin{figure}[tbph]
    \centering
    \includegraphics[width=0.85\linewidth]{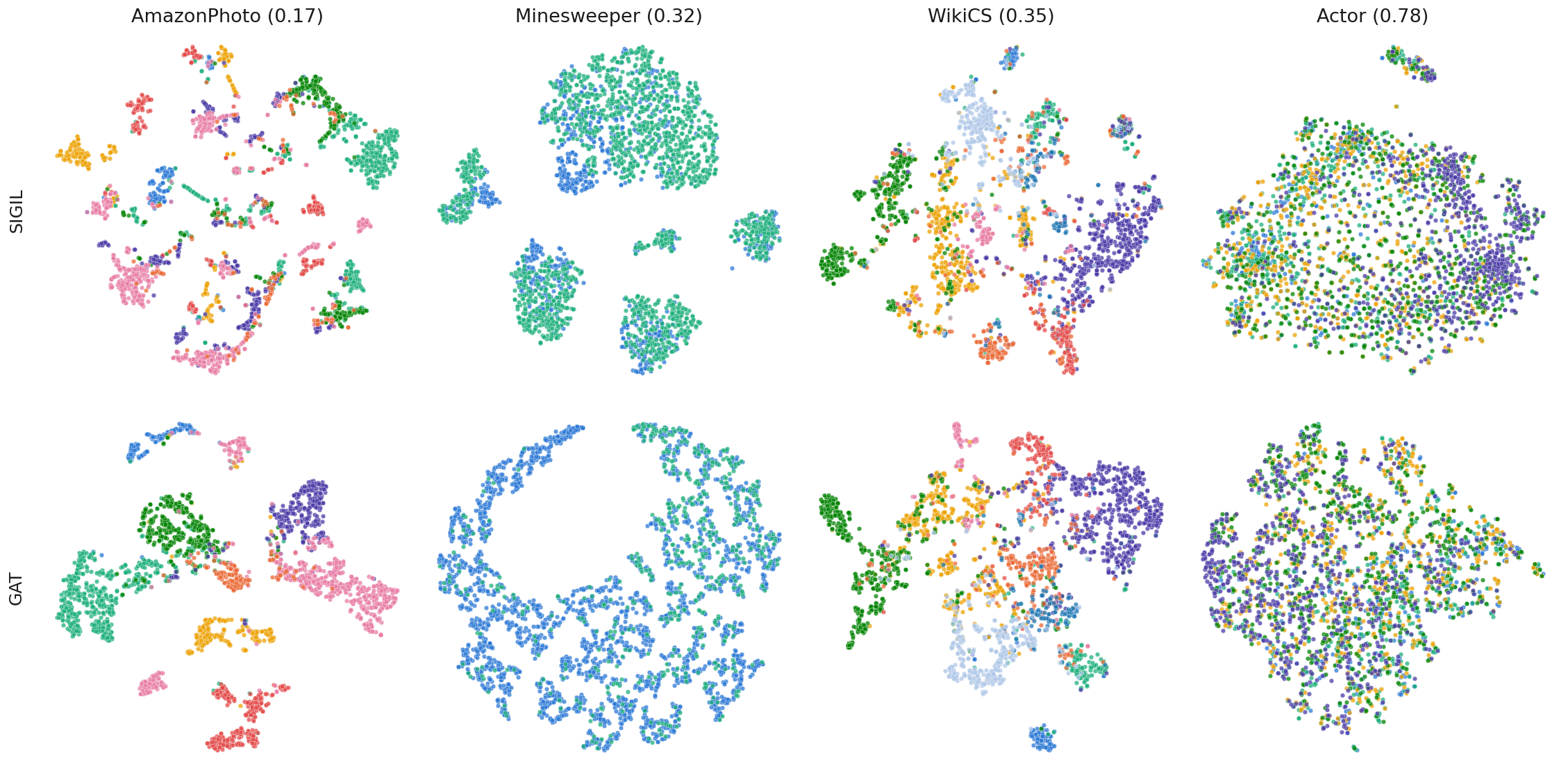}
    \caption{t-SNE plots of \SIGILNC learned space. Datasets are ordered based on heterophiliy.}
    \label{fig:embed_viz}
\end{figure}
\newpage
\section{Dataset Statistics}\label{sec:dataset_statistics}
\begin{table}[h]
\centering
\caption{Statistics and sources for the link prediction benchmarks used in §5.1.
Edge counts are undirected ($|E|$). For non-attributed graphs, node features are
initialized via DeepWalk on train edges and are not innate to the dataset.}
\label{tab:lp_dataset_stats}
\small
\begin{tabular}{lrrrrrl}
\toprule
Dataset & \#Nodes & \#Edges & Avg. deg. & Max. deg. & \#Features & Source \\
\midrule
\multicolumn{7}{l}{\textit{Attributed}} \\
Cora            &  2{,}708 &  5{,}278  &  3.90 &   168 & 1{,}433 & \cite{3045390.3045396} \\
CiteSeer        &  3{,}327 &  4{,}552  &  2.74 &    99 & 3{,}703 & \cite{3045390.3045396} \\
PubMed          & 19{,}717 & 44{,}324  &  4.50 &   171 &    500  & \cite{3045390.3045396} \\
AmazonComputers & 13{,}752 & 245{,}861 & 35.76 & 2{,}992 &   767 & \cite{shchur2019pitfallsgraphneuralnetwork} \\
AmazonPhoto     &  7{,}650 & 119{,}081 & 31.13 & 1{,}434 &   745 & \cite{shchur2019pitfallsgraphneuralnetwork} \\
CoauthorCS      & 18{,}333 &  81{,}894 &  8.93 &   136 & 6{,}805 & \cite{shchur2019pitfallsgraphneuralnetwork} \\
CoauthorPhysics & 34{,}493 & 247{,}962 & 14.38 &   382 & 8{,}415 & \cite{shchur2019pitfallsgraphneuralnetwork} \\
\midrule
\multicolumn{7}{l}{\textit{Non-attributed}} \\
C.ele           &    297  &  2{,}148  & 14.46 &   134 & -- & \cite{zhang2018link} \\
USAir           &    332  &  2{,}126  & 12.81 &   139 & -- & \cite{zhang2018link} \\
PB              &  1{,}222 & 16{,}714  & 27.36 &   351 & -- & \cite{zhang2018link} \\
NS              &  1{,}589 &  2{,}742  &  3.45 &    34 & -- & \cite{zhang2018link} \\
\bottomrule
\end{tabular}
\end{table}
\begin{table*}[h!]
  \caption{Statistics of the 44 knowledge graphs used in the KG reasoning
  experiment: 3 pretraining graphs, 23 inductive $(e,r)$ graphs (unseen entities
  and relations), and 18 inductive $(e)$ graphs (unseen entities).}
  \label{tab:kg_stats}
  \centering
  \scriptsize
  \setlength{\tabcolsep}{4pt}
  \begin{tabular}{@{}llrrrrrl@{}}
    \toprule
    \textbf{Group} & \textbf{Dataset} & \textbf{Ent.\ train} & \textbf{Ent.\ infer} & \textbf{Rel.\ train} & \textbf{Rel.\ infer} & \textbf{Test} & \textbf{Source} \\
    \midrule
    \multicolumn{8}{@{}l}{\emph{Pretraining}} \\
    Pretr. & FB15k-237       & 14{,}541 & 14{,}541 & 237   & 237   & 20{,}466 & \cite{toutanova-chen-2015-observed} \\
    Pretr. & WN18RR          & 40{,}943 & 40{,}943 & 11    & 11    & 3{,}134  & \cite{3504035.3504256} \\
    Pretr. & CoDEx-Medium    & 17{,}050 & 17{,}050 & 51    & 51    & 10{,}311 & \cite{safavi-koutra-2020-codex} \\
    \midrule
    \multicolumn{8}{@{}l}{\emph{Inductive $(e,r)$}} \\
    (e,r) & FB-25       & 5{,}190  & 4{,}097  & 163 & 216 & 5{,}716  & \cite{lee2023ingram} \\
    (e,r) & FB-50       & 5{,}190  & 4{,}445  & 153 & 205 & 3{,}879  & \cite{lee2023ingram} \\
    (e,r) & FB-75       & 4{,}659  & 2{,}792  & 134 & 186 & 3{,}106  & \cite{lee2023ingram} \\
    (e,r) & FB-100      & 4{,}659  & 2{,}624  & 134 & 77  & 2{,}329  & \cite{lee2023ingram} \\
    (e,r) & WK-25       & 12{,}659 & 3{,}228  & 47  & 74  & 1{,}131  & \cite{lee2023ingram} \\
    (e,r) & WK-50       & 12{,}022 & 9{,}328  & 72  & 93  & 3{,}225  & \cite{lee2023ingram} \\
    (e,r) & WK-75       & 6{,}853  & 2{,}722  & 52  & 65  & 1{,}144  & \cite{lee2023ingram} \\
    (e,r) & WK-100      & 9{,}784  & 12{,}136 & 67  & 37  & 4{,}496  & \cite{lee2023ingram} \\
    (e,r) & NL-0        & 1{,}814  & 2{,}026  & 134 & 112 & 763      & \cite{lee2023ingram} \\
    (e,r) & NL-25       & 4{,}396  & 2{,}146  & 106 & 120 & 744      & \cite{lee2023ingram} \\
    (e,r) & NL-50       & 4{,}396  & 2{,}335  & 106 & 119 & 859      & \cite{lee2023ingram} \\
    (e,r) & NL-75       & 2{,}607  & 1{,}578  & 96  & 116 & 607      & \cite{lee2023ingram} \\
    (e,r) & NL-100      & 1{,}258  & 1{,}709  & 55  & 53  & 793      & \cite{lee2023ingram} \\
    (e,r) & Metafam     & 1{,}316  & 656      & 28  & 28  & 184      & \cite{zhou2023multitaskperspectivelinkprediction} \\
    (e,r) & FBNELL      & 4{,}636  & 4{,}752  & 100 & 183 & 597      & \cite{zhou2023multitaskperspectivelinkprediction} \\
    (e,r) & MT1 tax     & 10{,}000 & 10{,}000 & 10  & 9   & 1{,}834  & \cite{zhou2023multitaskperspectivelinkprediction} \\
    (e,r) & MT1 health  & 10{,}000 & 10{,}000 & 7   & 7   & 1{,}566  & \cite{zhou2023multitaskperspectivelinkprediction} \\
    (e,r) & MT2 org     & 10{,}000 & 10{,}000 & 10  & 11  & 2{,}441  & \cite{zhou2023multitaskperspectivelinkprediction} \\
    (e,r) & MT2 sci     & 10{,}000 & 10{,}000 & 16  & 16  & 1{,}650  & \cite{zhou2023multitaskperspectivelinkprediction} \\
    (e,r) & MT3 art     & 10{,}000 & 10{,}000 & 45  & 45  & 3{,}113  & \cite{zhou2023multitaskperspectivelinkprediction} \\
    (e,r) & MT3 infra   & 10{,}000 & 10{,}000 & 24  & 27  & 2{,}405  & \cite{zhou2023multitaskperspectivelinkprediction} \\
    (e,r) & MT4 sci     & 10{,}000 & 10{,}000 & 42  & 42  & 1{,}388  & \cite{zhou2023multitaskperspectivelinkprediction} \\
    (e,r) & MT4 health  & 10{,}000 & 10{,}000 & 21  & 20  & 1{,}703  & \cite{zhou2023multitaskperspectivelinkprediction} \\
    \midrule
    \multicolumn{8}{@{}l}{\emph{Inductive $(e)$}} \\
    (e)   & FB v1       & 1{,}594  & 1{,}993  & 180 & 180 & 411      & \cite{teru2020inductive} \\
    (e)   & FB v2       & 2{,}608  & 4{,}145  & 200 & 200 & 947      & \cite{teru2020inductive} \\
    (e)   & FB v3       & 3{,}668  & 7{,}406  & 215 & 215 & 1{,}731  & \cite{teru2020inductive} \\
    (e)   & FB v4       & 4{,}707  & 11{,}714 & 219 & 219 & 2{,}840  & \cite{teru2020inductive} \\
    (e)   & WN v1       & 2{,}746  & 1{,}618  & 9   & 9   & 373      & \cite{teru2020inductive} \\
    (e)   & WN v2       & 6{,}954  & 4{,}011  & 10  & 10  & 852      & \cite{teru2020inductive} \\
    (e)   & WN v3       & 12{,}078 & 6{,}327  & 11  & 11  & 1{,}143  & \cite{teru2020inductive} \\
    (e)   & WN v4       & 3{,}861  & 12{,}334 & 9   & 9   & 2{,}823  & \cite{teru2020inductive} \\
    (e)   & NELL v1     & 3{,}103  & 833      & 14  & 14  & 201      & \cite{teru2020inductive} \\
    (e)   & NELL v2     & 2{,}564  & 4{,}586  & 88  & 88  & 935      & \cite{teru2020inductive} \\
    (e)   & NELL v3     & 4{,}647  & 8{,}048  & 142 & 142 & 1{,}620  & \cite{teru2020inductive} \\
    (e)   & NELL v4     & 2{,}092  & 7{,}073  & 76  & 76  & 1{,}447  & \cite{teru2020inductive} \\
    (e)   & ILPC Small  & 10{,}230 & 6{,}653  & 48  & 48  & 2{,}902  & \cite{galkin2022openchallengeinductivelink} \\
    (e)   & ILPC Large  & 46{,}626 & 29{,}246 & 65  & 65  & 10{,}184 & \cite{galkin2022openchallengeinductivelink} \\
    (e)   & HM 1k       & 36{,}237 & 9{,}899  & 11  & 11  & 476      & \cite{ijcai2017p250}7 \\
    (e)   & HM 3k       & 32{,}118 & 19{,}218 & 11  & 11  & 1{,}349  & \cite{galkin2022openchallengeinductivelink} \\
    (e)   & HM 5k       & 28{,}601 & 23{,}792 & 11  & 11  & 2{,}124  & \cite{galkin2022openchallengeinductivelink} \\
    (e)   & HM indigo   & 12{,}721 & 14{,}775 & 229 & 229 & 14{,}904 & \cite{NEURIPS2021_0fd600c9} \\
    \bottomrule
  \end{tabular}
\end{table*}
\begin{table}[h!]
\centering
\caption{Statistics and sources for the node classification datasets used in this paper.}\label{tab:datasets_stats}
\scriptsize
\setlength{\tabcolsep}{3pt}

\begin{tabular}{lrrrrrrrrr}\toprule
\textbf{Dataset} & \textbf{\#Nodes} & \textbf{\#Edges} & \textbf{\#Feature} & \textbf{\#Classes} & \textbf{\#Labeled Nodes} & \textbf{Train/Val/Test Ratios (\%)} & \textbf{Category} & \textbf{Source} \\\midrule
AirBrazil       & 131     & 1074       & 131     & 4  & 80   & 61.1/19.1/19.8   & Homophilic    & \cite{Ribeiro_2017}          \\
Cornell          & 183     & 554      & 1703    & 5  & 87   & 47.5/32.2/20.2   & Heterophilic  & \cite{Pei2020Geom-GCN:}             \\
Texas            & 183     & 558        & 1703    & 5  & 87   & 47.5/31.7/20.2   & Heterophilic  & \cite{Pei2020Geom-GCN:}             \\
Wisconsin        & 251     & 900        & 1703    & 5  & 120  & 47.8/31.9/20.3   & Heterophilic  & \cite{Pei2020Geom-GCN:}             \\
AirEurope           & 399     & 5995       & 399     & 4  & 80   & 20.1/39.8/40.1   & Homophilic    & \cite{Ribeiro_2017}          \\
AirUSA           & 1190    & 13599      & 1190    & 4  & 80   & 6.7/46.6/46.6    & Homophilic    & \cite{Ribeiro_2017}          \\
Chameleon        & 2277    & 36101      & 2325    & 5  & 1092 & 48.0/32.0/20.0   & Heterophilic  & \cite{Pei2020Geom-GCN:}             \\
Wiki             & 2405    & 17981      & 4973    & 17 & 340  & 14.1/42.9/43.0   & Homophilic    & \cite{Yang_2023}        \\
Cora             & 2708    & 10556      & 1433    & 7  & 140  & 5.2/18.5/36.9    & Homophilic    & \cite{3045390.3045396}           \\
Citeseer         & 3327    & 9104       & 3703    & 6  & 120  & 3.6/15.0/30.1    & Homophilic    & \cite{3045390.3045396}           \\
BlogCatalog      & 5196    & 343486     & 8189    & 6  & 120  & 2.3/48.8/48.8    & Homophilic    & \cite{Yang_2023}        \\
Squirrel         & 5201    & 217073     & 2089    & 5  & 2496 & 48.0/32.0/20.0   & Heterophilic  & \cite{Pei2020Geom-GCN:}             \\
Actor            & 7600    & 30019      & 932     & 5  & 3648 & 48.0/32.0/20.0   & Heterophilic  & \cite{Pei2020Geom-GCN:}             \\
AmzPhoto         & 7650    & 238162     & 745     & 8  & 160  & 2.1/49.0/49.0    & Homophilic    & \cite{shchur2019pitfallsgraphneuralnetwork}   \\
Minesweeper      & 10000   & 78804      & 7       & 2  & 5000 & 50.0/25.0/25.0   & Heterophilic  & \cite{platonov2023critical}\\
WikiCS           & 11701   & 431206     & 300     & 10 & 580  & 5.0/15.1/49.9    & Homophilic    & \cite{mernyei2022wikicswikipediabasedbenchmarkgraph}     \\
Tolokers         & 11758   & 1038000    & 10      & 2  & 5879 & 50.0/25.0/25.0   & Heterophilic  & \cite{platonov2023critical}\\
AmzComp          & 13752   & 491722     & 767     & 10 & 200  & 1.5/49.3/49.3    & Homophilic    & \cite{shchur2019pitfallsgraphneuralnetwork}   \\
DBLP             & 17716   & 105734     & 1639    & 4  & 80   & 0.5/49.8/49.8    & Homophilic    & \cite{bojchevski2018deep}  \\
CoCS             & 18333   & 163788     & 6805    & 15 & 300  & 1.6/49.2/49.2    & Homophilic    & \cite{shchur2019pitfallsgraphneuralnetwork}   \\
Pubmed           & 19717   & 88648      & 500     & 3  & 60   & 0.3/2.5/5.1      & Homophilic    & \cite{3045390.3045396}           \\
FullCora         & 19793   & 126842     & 8710    & 70 & 1400 & 7.1/46.5/46.5    & Homophilic    & \cite{bojchevski2018deep}  \\
Roman Empire     & 22662   & 65854      & 300     & 18 & 11331& 50.0/25.0/25.0   & Heterophilic  & \cite{platonov2023critical}\\
Amazon Ratings   & 24492   & 186100     & 300     & 5  & 12246& 50.0/25.0/25.0   & Heterophilic  & \cite{platonov2023critical}\\
CoPhysics        & 34493   & 495924     & 8415    & 5  & 100  & 0.3/49.9/49.9    & Homophilic    & \cite{shchur2019pitfallsgraphneuralnetwork}   \\
Questions        & 48921   & 307080     & 301     & 2  & 24460& 50.0/25.0/25.0   & Heterophilic  & \cite{platonov2023critical}\\
ogbn-arxiv            & 169343  & 1166243    & 128     & 40 & 90941& 53.7/17.6/28.7   & Homophilic    & \cite{hu2020ogb}                 \\

\bottomrule
\end{tabular}
\end{table}
}{}
\end{document}